%% file: main.tex
\documentclass{article}

\usepackage{microtype}
\usepackage{graphicx}
\usepackage{subcaption}
\usepackage{booktabs}
\usepackage{hyperref}
\usepackage{xurl}

\usepackage[preprint]{icml2026}

\usepackage{amsmath}
\usepackage{amssymb}
\usepackage{mathtools}
\usepackage{amsthm}
\usepackage[capitalize,noabbrev]{cleveref}
\usepackage{xcolor}

\input{macros}

\icmltitlerunning{A Sign Law for Entropy Drift under Sampler Mismatch in RLVR}

\begin{document}

\twocolumn[
  \icmltitle{A Sign Law for Entropy Drift under Sampler Mismatch in RLVR}

  \begin{icmlauthorlist}
    \icmlauthor{Zhiyuan Wang}{}
    \icmlauthor{Jacob Steeves}{}
    \icmlauthor{Chang Liu}{}
  \end{icmlauthorlist}
  % The review version (no [preprint]/[accepted] option) hides author information automatically.
  \icmlcorrespondingauthor{Zhiyuan Wang}{\mbox{zhiyuanalwang@gmail.com}}
  \icmlcorrespondingauthor{Jacob Steeves}{\mbox{jake@bittensor.com}}
  \icmlcorrespondingauthor{Chang Liu}{\mbox{liuchang.contact@gmail.com}}

  \icmlkeywords{reinforcement learning, RLVR, entropy, policy gradient, sampling}
  \vskip 0.3in
]

% No affiliations are listed: print the correspondence line without the empty affiliation list.
\makeatletter
\renewcommand{\printAffiliationsAndNotice}[1]{\global\icml@noticeprintedtrue%
  {\let\thefootnote\relax\footnotetext{\hspace*{-\footnotesep}\ificmlshowauthors #1\fi%
      {\raggedright Correspondence to: \icmlcorrespondingauthor@text.\par}
      \ \\
      \Notice@String
    }
  }
}
\makeatother
\printAffiliationsAndNotice{}

\input{sections/abstract}
\input{sections/introduction}
\input{sections/setup}
\input{sections/theory}
\input{sections/experiments}
\input{sections/practice}
\input{sections/related}
\input{sections/conclusion}

\section*{Impact Statement}
This paper presents work whose goal is to advance the field of machine learning.
It identifies an easily missed source of instability in reinforcement learning for language models and gives practitioners a diagnostic and corrections that can make training runs more reliable and reduce compute wasted on failed runs.
We do not foresee negative societal consequences specific to this work beyond those of the field.

\bibliography{refs}
\bibliographystyle{icml2026}

\newpage
\appendix
\onecolumn
\input{sections/appendix}

\end{document}

%% file: macros.tex
\newcommand{\pol}{\pi}                 % policy
\newcommand{\smp}{\mu}                 % sampler
\newcommand{\net}{N}                   % token-weighted net advantage
\newcommand{\gap}{\delta}              % sharpness gap
\newcommand{\KL}{\mathrm{KL}}
\newcommand{\E}{\mathbb{E}}
\newcommand{\Cov}{\mathrm{Cov}}
\newcommand{\keep}{S}                  % kept (sampled) support
\newcommand{\MASK}{\textup{\textsc{Mask}}}
\newcommand{\FAIL}{\textup{\textsc{Fail}}}
\newcommand{\NZ}{\textup{\textsc{NZ}}}
\newcommand{\PEN}{\textup{\textsc{Pen}}}
\newcommand{\GW}{\textup{\textsc{GapW}}}
\theoremstyle{plain}
\newtheorem{theorem}{Theorem}
\newtheorem{proposition}{Proposition}
\newtheorem{lemma}{Lemma}

\theoremstyle{definition}
\newtheorem{definition}{Definition}
\theoremstyle{remark}

%% file: sections/abstract.tex
\begin{abstract}
In reinforcement learning with verifiable rewards (RLVR), whether entropy rises or falls can depend on choices that matter little under exact sampling, such as whether over-length responses are masked or counted as failures.
We trace this to \emph{sampler mismatch}: rollouts are drawn with top-$k$/top-$p$ truncation or at a temperature the learner does not apply.
We prove that a constant advantage shift then changes the entropy of a softmax policy, at first order and in expectation, with a sign fixed by the sampler.
Positive shifts sharpen the policy under any sampler whose likelihood ratio to the policy increases with token probability, including top-$k$, top-$p$, min-$p$, and temperatures below one, and flatten it under flatter samplers.
In group-based RLVR with token-level losses, the baseline-dependent part of this drift is set by the \emph{token-weighted net advantage}, which masking and failing can push to opposite signs.
The rest comes from a covariance between advantages and the sampler's sharpness gap and feeds back on itself: in a two-level model, a push away from a sharper sampler takes off in finite time, while a pull toward it saturates.
On a 2B model, the entropy slope is linear in shifts of the net advantage under truncated sampling and indistinguishable from flat under exact sampling.
Token-level importance ratios leave 98\% of the drift, and four clipped minibatch updates per batch amplify it sevenfold.
In 100-step runs, a net-zero baseline only delays an entropy explosion, and entropy raised at the start keeps rising only under truncation; cancelling the first-order drift at every step keeps entropy bounded without changing the learner.
On a 9B reasoning model forked from an uncorrected run, zeroing the net advantage or renormalizing the learner raises held-out accuracy after ten steps by 8.2 and 8.8 points over the uncorrected rule, with no difference under exact sampling.
\end{abstract}

%% file: sections/introduction.tex
\section{Introduction}
\label{sec:intro}

\begin{figure*}[t]
\centering
\includegraphics[width=\textwidth]{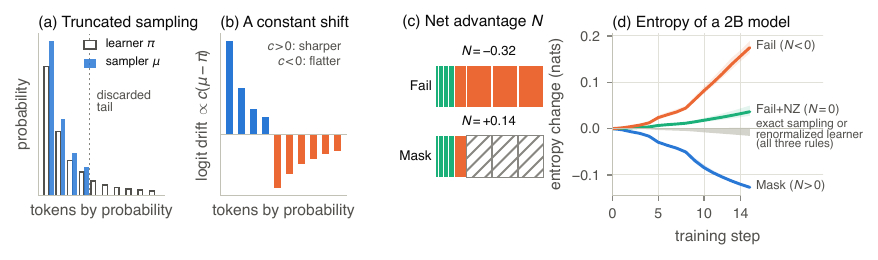}
\vskip -0.05in
\caption{\textbf{(a)} A truncated sampler $\smp$ renormalizes the learner's policy $\pol$ over the most probable tokens, so \textbf{(b)} a constant advantage shift $c$ moves the logits by $\eta c(\smp-\pol)$ in expectation, sharpening the policy for $c>0$ and flattening it for $c<0$. \textbf{(c)} In an illustrative group (width: tokens; hatched: masked but counted in the baseline), the shift a token-level loss applies is the net advantage $\net$. \textbf{(d)} Entropy change of a 2B model under truncated sampling (3 seeds); with exact sampling or a renormalized learner all rules stay in the gray band.}
\label{fig:teaser}
\vskip -0.1in
\end{figure*}

Policy entropy is one of the signals practitioners watch most closely in reinforcement learning with verifiable rewards (RLVR), and it can go wrong in either direction.
When it collapses, the policy stops exploring and accuracy plateaus~\citep{cui2025entropy,yu2025dapo}.
When it explodes, responses degrade and training can diverge~\citep{mai2026thinking,trl7627}.
Which way it moves often depends on choices that seem unrelated to entropy.
In the runs of \cref{fig:teaser}d, which differ only in how they treat over-length responses, masking these responses drives entropy down and counting them as failures drives it up.
With exact sampling, the same runs are indistinguishable.

Most accounts of entropy dynamics assume on-policy rollouts.
There, entropy follows the covariance between advantages and log-probabilities~\citep{cui2025entropy} and responds to clipping~\citep{park2025clip}, and baseline invariance makes a constant advantage shift harmless~\citep{williams1992reinforce}.
In particular, the constant part of an advantage shift has no first-order effect on entropy, so this theory gives no reason why masking and failing should push entropy in opposite directions.

In practice, engines often sample with top-$k$, top-$p$, or min-$p$ truncation, or at a temperature that the learner does not apply (\cref{sec:practice}).
Under such \emph{sampler mismatch}, a constant advantage shift adds a drift toward or away from the sampler (\cref{fig:teaser}a,b)~\citep{leroux2016efficient,leroux2025topr,marek2026score}.
Recent systems remove it by replaying the sampling mask~\citep{deepseek2025v32,mai2026thinking,xiaomi2026mimo,cognition2026swe17} or by centering the score~\citep{marek2026score}.
An automated review on a torchtitan pull request argues informally that under truncation the drift sharpens or flattens the policy and does not average out over a group~\citep{torchtitan5021}, and SCOPE-RL finds that positive samples raise entropy under high-temperature sampling~\citep{scoperl2025}.
We know of no sign statement for general samplers derived from the policy, and it has been unclear which batch quantity controls the drift in group-based RLVR, where advantages within a group already sum to zero, or what each correction removes.

We show that the sampler fixes the sign of the drift's effect on entropy, and that one batch statistic sets the part of the drift that depends on the baseline.
Our sign law (\cref{thm:sign}) covers every sampler sharper than the policy, including top-$k$, top-$p$, min-$p$, and temperatures below one: a positive advantage shift lowers entropy at first order and a negative shift raises it.
Flatter samplers, such as temperatures above one, reverse both directions.
A token-level loss shifts each group by its \emph{token-weighted net advantage} $\net_g$ (\cref{fig:teaser}c), which equals $\Cov(r,L)/\bar L$ under a group-mean baseline and gains a nonnegative term under masking (\cref{tab:contracts}).
Long responses fail more often, so failing them makes $\net$ negative and pushes entropy up, whereas in our runs masking makes it positive and pushes entropy down.
A second term arises because failed responses visit less confident prefixes.
It does not depend on the baseline and is smaller on real batches, but it feeds back on itself: in a two-level model, a push away from a sharper sampler accelerates and a pull toward it slows down (\cref{prop:feedback}).
Setting $\net=0$, as the length-weighted baseline of OPO~\citep{hao2025opo} does, removes the net-advantage term, and weighting the baseline by each response's gap cancels the two terms against each other.
Renormalizing the learner over the sampled support removes the whole drift and coincides with score centering, whereas token-level importance ratios only rescale it.

On a 2B model under truncated sampling, the entropy slope is linear in shifts of $\net$ and, across 13 rules, in the logged first-order drift.
The difference between masking and failing follows the sharpness of six samplers, reverses at temperature 1.2, and falls to 2--4\% under learner-side corrections, but not under importance ratios or multi-minibatch PPO.
In 100-step runs, zeroing $\net$ only delays an entropy explosion that exact sampling avoids (\cref{fig:long,fig:controls}).
At 9B, zeroing $\net$ or renormalizing the learner gains 8.2 and 8.8 held-out points over the uncorrected rule in ten steps under truncation, and nothing under exact sampling.
Our contributions are:
\begin{itemize}
\item A sign law for every sampler sharper or flatter than the policy (\cref{sec:position,sec:sign}).
\item A split of the mismatch drift into a net-advantage term and a gap covariance, an exactly solvable model of its asymmetric feedback, and what each correction removes (\cref{sec:aggregate,sec:fixes}).
\item Controlled evidence at 2B and 9B, including 100-step runs that separate the feedback, a sign check on OLMo-3-7B, and a survey of practice (\cref{sec:experiments,sec:practice}).
\end{itemize}

%% file: sections/setup.tex
\section{Setup}
\label{sec:setup}

\paragraph{Group-based RLVR with token-level losses.}
For each prompt $x$ the sampler draws a group of $G$ responses $y_1,\dots,y_G$ with lengths $L_i$ and verifiable rewards $r_i$.
A group baseline turns rewards into advantages, $A_i = r_i - \bar r$ with $\bar r=\frac1G\sum_j r_j$ (optionally divided by the group standard deviation; \citealp{shao2024deepseekmath}).
The policy $\pol_\theta$ is updated with the token-level surrogate
\begin{equation}
  \mathcal{J}(\theta) = \sum_{i} \omega_i A_i \sum_{t=1}^{L_i} \log \pol_\theta(y_{i,t}\mid s_{i,t}),
  \label{eq:surrogate}
\end{equation}
where $s_{i,t}=(x,y_{i,<t})$ is the prefix and $\omega_i\ge 0$ is the per-token weight that the aggregation rule assigns to response $i$.
Token-mean and prompt-mean aggregation give every token of a group the same weight ($\omega_i \propto 1/\sum_{j} L_j$), while sequence-mean-token-mean aggregation gives every response the same total weight ($\omega_i \propto 1/L_i$).
We follow the common practice of one optimizer step per batch, with old log-probabilities recomputed by the learner, so the ratio of proximal policy optimization (PPO) equals one and the gradient of \cref{eq:surrogate} is the policy gradient.
\Cref{sec:fixes} discusses importance ratios.

\paragraph{Over-length responses.}
Responses that reach the generation limit have no verifiable answer, and neither do a few that end early (about 30\% and 7\% of responses at 2B).
Pipelines treat both in one of three ways, all used throughout.
\MASK{}, a form of overlong filtering~\citep{yu2025dapo}, gives them zero reward when computing the baseline and then excludes them from the loss.
\FAIL{} gives them zero reward and trains on them.
\PEN{} additionally lowers the reward of long completed responses with a linear over-length penalty~\citep{yu2025dapo}.

\paragraph{Sampler mismatch.}
The learner differentiates $\log \pol_\theta(\cdot\mid s)=\log\mathrm{softmax}(z_\theta(s))$ at temperature one, but rollouts are drawn from a sampler $\smp(\cdot\mid s)$ derived from it.
Temperature scaling gives $\smp\propto\pol^{1/T}$.
Truncation keeps a set $\keep(s)$ of the most probable tokens and renormalizes, $\smp = \pol\,\mathbf 1_{\keep}/\pol(\keep)$; top-$k$, top-$p$, and min-$p$, applied in sequence as in vLLM~\citep{kwon2023vllm}, all produce such a set.
We say the pipeline has \emph{sampler mismatch} when $\smp\neq\pol$, and \emph{exact sampling} when $T=1$ and no truncation is applied.

\paragraph{Baseline invariance holds only without mismatch.}
Adding a constant $c$ to every advantage adds $c\sum_i\omega_i\sum_t\nabla\log\pol_\theta(y_{i,t}\mid s_{i,t})$ to the gradient.
At each visited prefix the sampled token has conditional expectation
\begin{equation}
  \E_{y\sim\smp}\!\big[c\nabla_\theta\log\pol_\theta(y\mid s)\big] = -c\nabla_\theta \KL\big(\smp(\cdot\mid s)\|\pol_\theta(\cdot\mid s)\big),
  \label{eq:drift}
\end{equation}
with $\smp$ held fixed.
It vanishes exactly when $\smp=\pol$, which is the classical baseline invariance~\citep{williams1992reinforce}.
Under mismatch it pulls the policy toward the sampler for $c>0$ and pushes it away for $c<0$, as prior work has noted~\citep{leroux2016efficient,leroux2025topr,arnal2025asymre,marek2026score}.

%% file: sections/theory.tex
\section{A Sign Law for Mismatch Drift}
\label{sec:theory}

The analysis is first order and uses plain gradient steps in the softmax parameterization, where every prefix has its own logits.
\Cref{sec:exp-exists} measures the same quantity in trained networks on held-out prompts.

\subsection{Entropy change at one position}
\label{sec:position}

Fix a prefix $s$ with logits $z$, policy $\pol=\mathrm{softmax}(z)$ over vocabulary $\mathcal V$, and entropy $H=-\sum_v\pol_v\log\pol_v$.
Sampling token $a$ with advantage $A$ and taking a gradient step of size $\eta$ on $A\log\pol_a$ changes the logits by $\eta A(e_a-\pol)$.
Since $\partial H/\partial z_v = -u_v$ with $u_v \coloneqq \pol_v(\log\pol_v+H)$, the entropy changes by
\begin{equation}
  \Delta H = -\eta A\,\Phi(a) + O(\eta^2), \qquad \Phi(a) \coloneqq u_a - \E_{v\sim\pol}[u_v].
  \label{eq:phi}
\end{equation}
Under exact sampling $\E_{a\sim\pol}[\Phi(a)]=0$, so a constant advantage has no first-order effect on entropy in expectation; this is the entropy-neutral identity behind on-policy analyses~\citep{cui2025entropy,stare2026}.
Under a sampler $\smp$ the same constant $c$ gives
\begin{equation}
  \E_{a\sim\smp}[\Delta H] = -\eta c\,\gap(\smp) + O(\eta^2),
  \label{eq:gap}
\end{equation}
where $\gap(\smp)\coloneqq\E_{a\sim\smp}[\Phi(a)] = \E_{\smp}[u]-\E_{\pol}[u]$.
We call $\gap(\smp)$ the \emph{sharpness gap} of the sampler.
It is zero for exact sampling and, as we show next, has a fixed sign whenever the sampler is ordered relative to the policy.

\subsection{The sign of the sharpness gap}
\label{sec:sign}

The sign of $\gap$ is determined by how the sampler reweights tokens relative to the policy.

\begin{definition}
A sampler $\smp$ is \emph{sharper} than $\pol$ if $\smp_u/\pol_u \ge \smp_v/\pol_v$ whenever $\pol_u>\pol_v$, and \emph{flatter} if the reverse inequality holds.
\end{definition}

Truncation to any set $\keep$ that contains every token more probable than its least probable member is sharper than $\pol$; this covers top-$k$, top-$p$, min-$p$, and their sequential composition.
Temperature $T<1$ is sharper, temperature $T>1$ and mixing with the uniform distribution are flatter, and temperature $T\le 1$ followed by truncation is sharper.
Temperature $T>1$ followed by truncation is neither, and the sign of its gap depends on $\pol$.

\begin{theorem}[Sign law]
\label{thm:sign}
Let $\pol$ be a non-uniform distribution with full support and let $\smp\neq\pol$.
If $\smp$ is sharper than $\pol$, then $\gap(\smp)>0$; if $\smp$ is flatter than $\pol$, then $\gap(\smp)<0$.
\end{theorem}

\emph{Proof idea.}
A truncation to a set $\keep$ of the most probable tokens has $\gap=\pol(\keep^c)(\E_\pol[u\mid\keep]-\E_\pol[u\mid\keep^c])>0$, which follows by comparing $\keep$ with the nested set $\{v:\pol_v\ge e^{-H}\}$ on which $u\ge0$.
A sharper sampler is a mixture of such truncations over the level sets of $\smp/\pol$, and $\gap$ is linear in $\smp$; flatter samplers mix over the complements, which flips the sign (\cref{app:proofs}). \hfill$\square$

By \cref{eq:gap}, a constant shift $c$ therefore lowers entropy at first order under a sharper sampler when $c>0$ and raises it when $c<0$, and the directions reverse under a flatter sampler.

\subsection{From positions to batches: the net advantage}
\label{sec:aggregate}

In group-based RLVR the advantages of each group sum to zero, so it is not obvious that any constant shift is present.
The token-level loss, however, weights response $i$ by $\omega_iL_i$, and the shift it applies is the length-weighted mean advantage of each group.

\begin{definition}
For group $g$, let $W_g=\sum_{i\in g}\omega_iL_i$. The \emph{token-weighted net advantage} of the group is $\net_g = \sum_{i\in g}\omega_iL_iA_i/W_g$, and that of the batch is $\net=\sum_gW_g\net_g/\sum_gW_g$.
\end{definition}

Write $\gap_{i,t}$ for the sharpness gap at prefix $s_{i,t}$, $\bar\gap_i$ for its mean over response $i$, and $\bar\gap_g$ and $\Cov_g$ for the $\omega_iL_i$-weighted mean and covariance over group $g$.

\begin{proposition}[Decomposition]
\label{prop:aggregate}
Treat the token weights $\omega_i$ and any baseline as fixed given the group, ignoring the $O(1/G)$ dependence of data-dependent weights and baselines on each sampled response (masking adds an $O(1)$ dependence, discussed below). In the softmax parameterization, the first-order change of entropy summed over the visited prefixes is $\Delta H_{\mathrm{tot}}=D_{\mathrm{net}}+D_{\mathrm{cov}}+D_{\mathrm{on}}$, with
\begin{equation}
\begin{aligned}
  D_{\mathrm{net}}&=-\eta\textstyle\sum_g W_g\,\net_g\,\bar\gap_g,\\
  D_{\mathrm{cov}}&=-\eta\textstyle\sum_g W_g\Cov_g(A,\bar\gap),\\
  D_{\mathrm{on}}&=-\eta\textstyle\sum_i\omega_iA_i\sum_t\big(\Phi_{i,t}(y_{i,t})-\gap_{i,t}\big).
\end{aligned}
\label{eq:aggregate}
\end{equation}
Shifting the advantages of a group by a constant changes only $D_{\mathrm{net}}$, up to a term with zero mean given each prefix. Under exact sampling, $\gap\equiv0$ and $D_{\mathrm{net}}=D_{\mathrm{cov}}=0$.
\end{proposition}

The sum $D_{\mathrm{net}}+D_{\mathrm{cov}}$ is the \emph{mismatch drift}, and $D_{\mathrm{on}}$ contains the on-policy covariance between advantages and token probabilities~\citep{cui2025entropy}.
Only $D_{\mathrm{net}}$ depends on the baseline.
Under a sharper sampler each group contributes to it with the sign of $-\net_g$, so the net advantage predicts the direction of the baseline-controlled drift.
$D_{\mathrm{cov}}$ arises when advantages correlate with gaps.
Its sign is not fixed in general (\cref{app:proofs}), but in our runs failed responses visit less confident prefixes (\cref{sec:exp-exists}), so it raises entropy under every rule we test.
On 9B batches, $D_{\mathrm{net}}$ carries 66--78\% of the mismatch drift under \FAIL{} and \PEN{} (\cref{app:more}).

With a group-mean baseline, $\net_g=\Cov_g(r,L)/\bar L_g$ (divided by $\sigma_g$ under standard-deviation normalization), which is negative under \FAIL{} and \PEN{} because long responses fail more often.
Sequence-mean aggregation makes the batch-level $\net$ zero (\cref{tab:contracts}).
\MASK{} adds $(n_{\mathrm{mask}}/G)\bar r_M\ge0$, where $n_{\mathrm{mask}}$ is the number of masked responses and $\bar r_M$ the mean reward over the unmasked set $M$, because masked responses lower the baseline but carry no negative advantage.
In our runs $\net=+0.30$ at 2B and $+0.16$ at 9B, so choosing between masking and failing flips the sign of the baseline-controlled drift.
It would not if masked responses were dropped before computing the baseline.
Because a response is masked exactly when it reaches the limit, masking also tilts a shift toward responses that finish.
This happens even under exact sampling, but there \MASK$-$\FAIL{} is only $+0.7$, against $-21.3$ under truncation (\cref{sec:exp-exists}).
The length-weighted baseline $b_g=\sum_jL_jr_j/\sum_jL_j$ of OPO~\citep{hao2025opo} sets $\net_g=0$, and we write \NZ{} for applying it.

\paragraph{Feedback through the gap.}
Both mismatch terms scale with the gap, which itself depends on the policy.
A two-level model solves the resulting dynamics exactly.

\begin{proposition}[Asymmetric feedback]
\label{prop:feedback}
Let a prefix have $k<V/2$ tokens with logit $\Lambda>0$ and $V-k$ tokens with logit $0$, sampled by top-$k$; let $x=\pol(\keep^c)$ be the discarded mass and $\kappa=1/k+1/(V-k)$.
Under the expected update $\dot z=\eta c(\smp-\pol)$ for a constant advantage $c$, the logits keep this form, and
\begin{enumerate}
\item[(i)] $\gap=\kappa\Lambda x^2(1-x)$ and $\dot x=-\eta c\kappa x^2(1-x)$;
\item[(ii)] if $c<0$, $x$ reaches $1/2$ at the finite time $t_{1/2}=\big(1/x_0-2+\log\frac{1-x_0}{x_0}\big)/(\eta|c|\kappa)$; if $c>0$, $x$ decays only as $\Theta(1/t)$;
\item[(iii)] if the on-policy term also raises $\Lambda$ at rate $\rho>0$, a push $c<0$ runs away when $x>\rho/(\eta|c|\kappa)$ and is reversed when $x<\rho/(\eta|c|\kappa)$.
\end{enumerate}
\end{proposition}

By (i), $|\dot x|$ grows with $x$ as long as less than two thirds of the mass is discarded, so a push away from a sharper sampler speeds up, while a pull toward it slows down.
Under \NZ{}, failed responses visit prefixes with larger gaps, so one class of prefixes is pushed and another is pulled.
Their summed entropy then eventually rises even if it first falls, provided the pulled class starts near its floor.
A constant compensation smaller than the push only delays the rise, whereas cancelling the first-order drift at every step keeps total entropy constant at first order (\cref{app:proofs}, \cref{fig:feedback}).
Top-$p$ caps the discarded mass at $1-p$, and hence the $\ell_1$ norm of the per-step logit drift at $2\eta|c|(1-p)$, whereas under top-$k$ the drift grows with the discarded mass.
Adam's diagonal preconditioner can change the sign at individual prefixes, although the pure-shift probes of \cref{sec:exp-exists} have the same sign under Adam and plain gradient steps.

\subsection{What removes the drift}
\label{sec:fixes}

\Cref{prop:aggregate} suggests three kinds of fixes.
One can zero the net advantage, which removes the baseline-controlled term, choose the baseline so that the two mismatch terms cancel, or remove the mismatch on the learner's side.

\begin{proposition}[Corrections]
\label{prop:fixes}
Let $\smp=\pol\mathbf 1_\keep/\pol(\keep)$ be a truncation.
\begin{enumerate}
\item[(a)] \emph{Net-zero advantages.} \NZ{} sets every $\net_g=0$ and removes the net-advantage drift; the gap covariance and the on-policy term remain.
\item[(b)] \emph{Gap-weighted baseline.} The group baseline $b_g=\sum_{i\in g}\omega_iL_i\bar\gap_ir_i\big/\sum_{i\in g}\omega_iL_i\bar\gap_i$, which we call \GW{}, makes $\sum_{i\in g}\omega_iL_iA_i\bar\gap_i=0$ in every group, so the two mismatch terms cancel: $D_{\mathrm{net}}+D_{\mathrm{cov}}=0$ at first order, with a small net advantage offsetting the gap covariance. Since $\bar\gap_i\ge0$ under a sharper sampler, $b_g$ is a weighted mean of the group's rewards.
\item[(c)] \emph{Support renormalization equals score centering.} If the learner renormalizes over the sampled support, $\log\pol_\keep(y)=\log\pol(y)-\log\pol(\keep)$, then $\nabla\log\pol_\keep(y)=\nabla\log\pol(y)-\E_{v\sim\smp}[\nabla\log\pol(v)]$, which is the centered score of \citet{marek2026score}. The learner's distribution then equals $\smp$, and both mismatch terms vanish.
\item[(d)] \emph{Token-level importance ratios rescale the drift.} The exact ratio $\pol/\smp$ equals $\pol(\keep)$ on $\keep$, so a constant shift $c$ gives $\E_{y\sim\smp}[\pol(\keep)\,c\nabla\log\pol(y)]=\pol(\keep)\cdot\E_{\smp}[c\nabla\log\pol]$: each prefix's drift is multiplied by its kept mass and keeps its sign, and no other weighting of sampled tokens removes it (\cref{app:proofs}).
\end{enumerate}
\end{proposition}

\GW{} reduces to \NZ{} when gaps are equal.
It needs the same learner-side logits as renormalization, which removes the drift at every prefix and is unbiased for the truncated policy, whereas \GW{} cancels only the first-order entropy effect of each group.
We therefore use \GW{} mainly to isolate $D_{\mathrm{cov}}$.
Part (c) restates the sampling-mask replay of DeepSeek-V3.2~\citep{deepseek2025v32} (\cref{app:sc}).
For temperature mismatch, the corresponding fix divides the learner's logits by the sampling temperature.
Part (d) is specific to truncation.
A tempered sampler has full support, so exact ratios remove its drift, although clipping reintroduces part of it~\citep{marek2026score,han2026baselines}.
Truncated importance sampling~\citep{yao2025tis} inherits (d), sequence-level ratios~\citep{zheng2025gspo} can reverse the drift, and log-probabilities computed before truncation, which vLLM returns by default, give a ratio of one (\cref{app:proofs,app:stale}).
None of these corrections freezes entropy, since $D_{\mathrm{on}}$ remains.

%% file: sections/experiments.tex
\section{Experiments}
\label{sec:experiments}

We test the predictions of \cref{sec:theory} in turn, ending with accuracy at 9B and 100-step runs at 2B.

\subsection{Setup}
\label{sec:exp-setup}

\paragraph{Models and data.}
We use Qwen3.5-2B (non-thinking mode, released checkpoint), MiMo-V2.6-Distill-Qwen-9B~\citep{xiaomi2026mimo}, and, as a second model family, OLMo-3-7B-Instruct~\citep{olmo2025olmo3}.
The 9B experiments fork from step 20 (weights and optimizer state) of an earlier 9B run with the same trainer and the \FAIL{} rule, which sampled with top-$k$ 20 and top-$p$ 0.95 and did not renormalize the learner.
Prompts come from the MATH training set~\citep{hendrycks2021math} (7{,}496 problems, graded by math-verify); 750 are held out from all training, including the earlier 9B run, and supply all evaluation and probe prompts.

\paragraph{Training.}
Unless a rule says otherwise, all runs use one trainer: groups of 8 responses, a group-mean baseline without standard-deviation normalization, prompt-mean aggregation, AdamW with learning rate $10^{-6}$, gradient clipping at 1, no KL or entropy terms, and one optimizer step per batch.
The 2B runs use binary rewards, 16 groups per step, and a 2{,}048-token limit for 15 steps (100 in \cref{sec:exp-long}); the 9B runs add the \PEN{} penalty and use 32 groups per step and an 8{,}192-token limit for 10 steps.

\paragraph{Samplers and rules.}
We compare truncated sampling (top-$k$ 20, then top-$p$ 0.95, at $T=1$), top-$p$ 0.99 and top-$p$ 0.8 alone, exact sampling, and untruncated sampling at $T=0.7$ and $T=1.2$, with a learner at $T=1$.
As corrections we test support renormalization (\cref{app:details}), token-level importance ratios against the truncated distribution, exact or inside the clipped objective, and a learner that applies the sampling temperature.
At 2B the rules are \MASK{}, \FAIL{}, \FAIL+\NZ{}, and \FAIL+\GW{}, with sequence-mean aggregation, a fixed or per-batch shift, standard-deviation normalization, four clipped minibatch updates per batch, and SGD as controls.
At 9B the rule is \FAIL+\PEN{}, alone, with \NZ{}, or with a renormalized learner.
The 2B runs log the learner's gap for every sampled response.

\paragraph{Measuring entropy.}
Before every update we compute the mean full-softmax token entropy of the policy on a fixed probe set: one response for each of 64 held-out prompts, sampled once from the initial policy and scored by teacher forcing.
The \emph{entropy slope} of a run is the least-squares slope of this quantity over steps.
Runs with the same seed see the same prompts and sampling seeds.
We use two or three seeds per 2B cell and four per 9B cell, and treat the seed as the unit of replication.

\subsection{The net advantage predicts the direction of entropy}
\label{sec:exp-exists}
\label{sec:exp-predicts}

At the 9B fork the mean gap under the truncated sampler is $4.4\times10^{-3}$ per token, larger on completed incorrect than on correct responses ($5.5$ against $4.0\times10^{-3}$), which makes $D_{\mathrm{cov}}$ raise entropy.
A pure positive shift lowers held-out entropy at both scales, along the raw gradient and the Adam step; under exact sampling the effect is only 4\% (9B) or under 6\% (2B) as large (\cref{tab:probes}).

Under truncated sampling the three rules move entropy in the directions their net advantages predict (\cref{fig:teaser,tab:e4}).
\MASK{} ($\net\approx+0.30$) lowers it, \FAIL{} ($\net\approx-0.08$) raises it, and \FAIL+\NZ{} lies in between; the paired \MASK$-$\FAIL{} difference ranges from $-20.5$ to $-22.0\times10^{-3}$ nats per step across seeds.
Removing the mismatch from either side makes the ordering vanish: the difference falls to $-3\%$ of its truncated value with exact sampling and to $-4\%$ with support renormalization.
Since the sampler still truncates under renormalization, truncation by itself is not the cause.

\begin{figure*}[t]
\centering
\includegraphics[width=\textwidth]{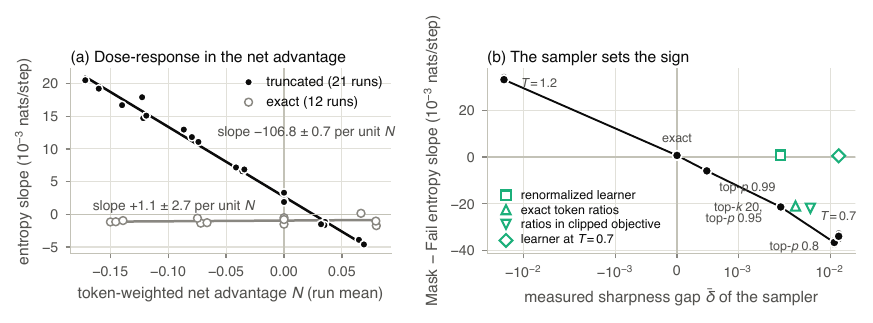}
\vskip -0.05in
\caption{\textbf{(a)} Dose-response on 2B: advantages of group $g$ shifted by $-\lambda\net_g$, one point per run. \textbf{(b)} Paired \MASK$-$\FAIL{} slope difference against the measured sharpness gap (symmetric log scale; small dots: seeds; temperatures untruncated with a learner at $T=1$; open green markers: corrections, offset horizontally from the gap of their sampler for visibility).}
\label{fig:sign}
\vskip -0.1in
\end{figure*}

A dose-response experiment varies the baseline-controlled term alone (\cref{fig:sign}a).
Starting from \FAIL{}, we subtract $\lambda\net_g$ from every advantage in group $g$.
This is a constant shift within each group, so by \cref{prop:aggregate} it changes only $D_{\mathrm{net}}$; $\lambda=1$ gives \NZ{}.
Under truncated sampling the entropy slope is linear in $\net$ (\cref{tab:dose}), at $(-106.8\pm0.7)\times10^{-3}$ nats per step per unit of $\net$ (21 runs; mean and s.e.\ over per-seed fits).
Under exact sampling it is $(+1.1\pm2.7)\times10^{-3}$, indistinguishable from zero.
This line describes shifts of a single rule, and \MASK{} ($\net=+0.30$) lies well above it ($-9.4$ against about $-29$), for two reasons.
First, masking removes the over-length responses, whose gaps are 1.5 times those of completed ones, so each unit of $\net$ carries less drift.
The logged drift $D_{\mathrm{net}}+D_{\mathrm{cov}}$ reflects this: across 35 runs of 13 rules under the truncated sampler, the 15-step entropy slope is linear in it with correlation 0.97 (\cref{fig:drift} in \cref{app:more}).
Second, under Adam the response saturates at large drift.
Shifting the advantages of \MASK{} to $\net=+0.15$ changes its slope only from $-9.8$ to $-8.6$, whereas $\net=0$ gives $+0.5$ and $\net=-0.24$ gives $+17.3$.
Under SGD at a matched rate, the slope halves along with the logged drift ($-8.3$ to $-4.3$; \cref{app:more}), so the saturation comes from Adam's per-parameter scaling.

The gap covariance moves entropy too, though less.
At $\net=0$ the truncated slope is still positive ($+2.4$ for \FAIL+\NZ{}), in line with its logged $D_{\mathrm{cov}}$ (\cref{fig:drift}), while exact sampling ($-1.0$) and support renormalization ($-0.3$) show no rise.
Raising $D_{\mathrm{cov}}$ at fixed $\net=0$, by adding to each advantage a multiple of the response's standardized mean log-probability, raises entropy more under truncated than under exact sampling (interaction $+4.0$ and $+3.2$; \cref{app:more}).
Over 15 steps \GW{}, which cancels both terms, brings the truncated slope to $-1.5$, the level of exact sampling and of a fixed shift with the same net advantage.

Sequence-mean aggregation, which makes $\net=0$ for \FAIL{} but not for \MASK{}, brings the \FAIL{} slope to the level of \FAIL+\NZ{} and leaves \MASK{} unchanged, as predicted out of sample (\cref{app:more}).

\subsection{The sampler sets the sign, and importance ratios do not remove it}
\label{sec:exp-sign}

The difference between masking and failing follows the measured sharpness gap of the sampler across six samplers (\cref{fig:sign}b).
At $T=1.2$, the only flatter sampler, the gap is negative and the rules reverse in all three seeds (\cref{tab:temp}).
\MASK{} raises entropy ($+21.5$), \FAIL{} lowers it ($-11.6$), and \FAIL+\NZ{} lies in between ($-6.6$); the \MASK$-$\FAIL{} difference is $+32.4$, $+32.5$, and $+34.3\times10^{-3}$ nats per step.
Among truncations the difference grows with the gap, monotonically but concavely, from $-5.9$ for top-$p$ 0.99 to $-36.5$ for top-$p$ 0.8 over a 23-fold range of gaps.
Even top-$p$ 0.99 retains 28\% of the effect of top-$k$ 20 with top-$p$ 0.95, and across sampler families the size is not a function of the gap alone.

Among the corrections tested in 15-step runs, only those that remove the mismatch on the learner's side remove the drift (\cref{fig:sign}b, green markers).
Support renormalization reduces the difference to $-4\%$ of its truncated value, and a learner that applies the sampling temperature reduces it at $T=0.7$ to 2\%.
Exact token-level importance ratios leave 98\% of it, consistent with \cref{prop:fixes}d and a mean kept mass of 0.99, and inside the clipped objective the same ratios leave 103\%.
Multi-minibatch PPO, with four clipped updates per batch against the sampling-time policy, amplifies the difference sevenfold per batch under truncation ($-139$ and $-154$, about 1.7 times per update); under exact sampling 1\% of that remains.
Under standard-deviation normalization the difference is $-20.2$ and $-17.3$, and $-4\%$ of that under exact sampling.
OLMo-3-7B-Instruct shows the same signs with a logged drift of similar size but about an eighth of the entropy response, so the conversion from drift to entropy depends on the model (\cref{app:more}).

\subsection{Accuracy at 9B}
\label{sec:exp-acc}

\begin{table}[t]
\caption{Accuracy (\%) of the 9B model 10 steps after the fork (78.7\% held-out), mean over 4 seeds, evaluated with the truncated sampler. Held-out: 300 problems $\times1$; fixed: 32 $\times8$; bench.: mean of four benchmarks (\cref{tab:e6full}). $\Delta H$: change of probe-set entropy. Renorm.: renormalized learner.}
\label{tab:e6}
\vskip 0.05in
\centering
\footnotesize
\setlength{\tabcolsep}{3pt}
\begin{tabular}{@{}llrrrr@{}}
\toprule
Sampler & Rule & Held-out & Fixed & Bench. & $\Delta H$ \\
\midrule
Trunc. & \FAIL+\PEN{} & 75.3 & 72.3 & 56.4 & $+0.215$ \\
 & \quad+\NZ{} & 83.5 & 83.6 & 67.6 & $+0.016$ \\
\midrule
Exact & \FAIL+\PEN{} & 84.0 & 83.3 & 66.6 & $+0.007$ \\
 & \quad+\NZ{} & 84.2 & 82.7 & 66.3 & $+0.009$ \\
\midrule
Renorm. & \FAIL+\PEN{} & 84.2 & 81.3 & 66.4 & $+0.035$ \\
\bottomrule
\end{tabular}
\vskip -0.1in
\end{table}

Under truncated sampling (\cref{tab:e6}), \FAIL+\PEN{} ($\net\approx-0.07$) raises probe entropy by 0.215 nats over ten steps, and its held-out accuracy falls from 78.7\% to 75.3\%.
With \NZ{} the entropy change is 0.016 nats and accuracy rises to 83.5\%, the level that both rules reach under exact sampling.
The difference of $+8.2$ points is positive in all four seeds ($+6.0$ to $+11.7$; $p=0.007$ with the seed as the unit), its interval excludes zero on every benchmark (\cref{app:more}), and the two arms do not differ in over-length rate (0.19 and 0.20).
Under exact sampling the same change leaves held-out accuracy unchanged ($+0.2$ points), although one benchmark interval, for AMC23, excludes zero at $-2.2$; the interaction is $+7.9$ points on the held-out set ($[4.7,11.1]$, $p=0.004$).
Renormalizing the learner reaches the same accuracy without changing the rule: 84.2\%, $+8.8$ points over the uncorrected learner ($[2.6,15.1]$, positive in every seed) and within one point of exact sampling and of \NZ{}.
All arms start from a fork trained for 20 steps with the uncorrected rule, so part of every corrected arm's gain may be recovery from that drift, and the fork's accuracy is a single evaluation (s.e.\ about 2.4 points).
The \NZ{} effect persists when the optimizer state is reset at the fork (\cref{app:more}), but over a hundred steps at 2B, zeroing $\net$ is not enough (\cref{sec:exp-long}).

\subsection{Over long horizons the gap covariance feeds back}
\label{sec:exp-long}

\begin{figure*}[t]
\centering
\includegraphics[width=\textwidth]{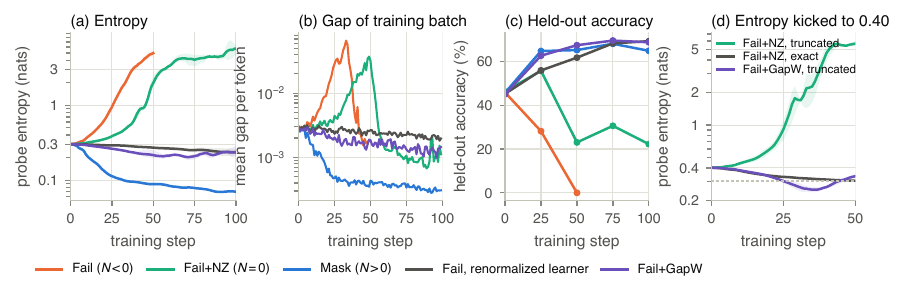}
\vskip -0.05in
\caption{Long 2B runs under truncated sampling (top-$k$ 20, top-$p$ 0.95; \FAIL{} to step 50; \cref{tab:long}). \textbf{(a)} Probe-set entropy (mean over seeds; bands: range). \textbf{(b)} Mean gap per token of the training batch (renormalized learner: untruncated softmax). \textbf{(c)} Held-out accuracy. \textbf{(d)} 50-step runs from logits divided by 1.2 (two seeds; dotted: initial entropy without the kick).}
\label{fig:long}
\vskip -0.1in
\end{figure*}

We run five 2B rules for 100 steps under truncated sampling with three seeds (two for the renormalized learner; \FAIL{} stops at step 50, once its entropy has saturated) and evaluate held-out accuracy every 25 steps (\cref{fig:long}).
\FAIL{} drives entropy above 0.5 nats at step 17 in all seeds, and its held-out accuracy falls from 45.8\% to 28.1\% at step 25 and to zero at step 50.
\FAIL+\NZ{} follows 17--18 steps later: it crosses 0.5 at steps 34--35 and reaches 4.5--7.3, while held-out accuracy drops from 55.7\% at step 25 to 23.0\% at step 50.
In both rules the batch gap rises from $3\times10^{-3}$ to $(4$--$8)\times10^{-2}$ before entropy nears its maximum (\cref{fig:long}b).
\MASK{} moves the other way and slows down, with its gap falling to $(2$--$3)\times10^{-4}$ and its entropy to 0.07 by step 100.
Adding to \FAIL+\NZ{} a fixed shift of $+0.05$, the 100-step mean net advantage of \GW{}, first makes the first-order drift negative, and entropy falls to 0.19--0.21 by steps 40--45.
The gap covariance then outgrows the net-advantage term, from 0.25 to 7.6 times its size in one seed and from 0.9 to 3.1 in the other, and entropy turns, reaching 0.87 and 0.40 by step 100, as \cref{prop:feedback} predicts.

Two controls tie this rise to the mismatch (\cref{fig:controls} in \cref{app:more}).
Under exact sampling \FAIL+\NZ{} never exceeds its initial entropy and reaches 68--71\% held-out accuracy.
Under top-$p$ 0.95 without top-$k$ the gap still doubles, but entropy peaks at 1.4--1.7 and then falls; the cap on the discarded mass bounds the per-step drift but does not by itself explain the fall.
To separate the feedback from a rise of entropy by itself, we raise the initial entropy by dividing the logits by 1.2 (probe entropy $0.30\to0.40$) and train \FAIL+\NZ{} for 50 steps (\cref{fig:long}d).
Under truncation entropy then crosses 1.0 at steps 25 and 28, against 43--46 from the unmodified start.
Under exact sampling it falls back to 0.30 and held-out accuracy reaches 64--68\%, so the raised entropy keeps rising only when the sampler is mismatched.

Cancelling the first-order drift at every step prevents the explosion within 100 steps without changing the learner.
With \GW{}, which cancels the drift group by group, entropy stays between 0.18 and 0.31 nats in three seeds, and it falls back from the raised start.
With a single shift per batch, recomputed so that the logged drift of the batch is zero, entropy ends at 0.14--0.20.
Held-out accuracy at step 100 is 68.7\%, 70.4\%, and 69.2\% for \GW{}, the per-batch shift, and the renormalized learner.
A fixed shift of $+0.13$ on \FAIL{}, as large as the mean per-batch shift, also keeps entropy bounded (final entropy 0.15 and 0.27; held-out 70.7\% and 67.7\%).
Its net advantage starts at the level of \NZ{}$+0.05$, which failed, but grows to about $+0.10$ as over-length failures become rarer, as do the net advantages of \GW{} and the per-batch shift.
\Cref{prop:feedback} requires a compensation that grows with the push, but 100 steps cannot tell growth from level.
Unlike a fixed shift, the adaptive corrections need no tuning.

%% file: sections/practice.tex
\section{Sampler Mismatch in Practice}
\label{sec:practice}

In a survey of 19 trainers in 18 RL frameworks, three inference engines, and 123 technical reports and public recipes (\cref{app:survey}), mismatch enters through framework defaults, inherited generation configs, and recipes; ROLL, for example, samples with $T=0.95$, top-$p$ 0.7, and top-$k$ 50 without dividing by $T$.
Of the 55 entries that disclose their training sampler, 7 (13\%) train with it uncorrected, and several frameworks added corrections only during 2026.
In single 300-step runs of a 1.5B model, entropy over the last 20 steps was 3.63 with top-$p$ 0.95 and top-$k$ 512, 0.69 with top-$p$ 1.0, and 0.54 after renormalization, although the uncorrected arm already used importance ratios against the truncated distribution~\citep{trl7627}.

We suggest four practices.
First, remove the mismatch on the learner's side: renormalize over the sampled support or center the score on the post-truncation distribution, and apply the sampling temperature in the learner, as a practitioner guide also recommends~\citep{fang2025praxis}.
Second, if only the advantages can be changed, cancel the first-order drift at every step with \GW{} or a per-batch shift; zeroing $\net$ is not enough, and a fixed shift needs hand tuning.
Third, do not rely on importance ratios or clipping under truncation, which left 98\% of the drift after one update and amplified it sevenfold over four.
Fourth, log $\net$ and the mean gap.

%% file: sections/related.tex
\section{Related Work}
\label{sec:related}

\paragraph{Entropy dynamics in RLVR.}
Entropy collapse is addressed with clip-higher~\citep{yu2025dapo}, entropy-targeted objectives~\citep{fade2026,entrocraft2026,qae2025,stare2026}, and adaptive controllers~\citep{zhang2025aer,yang2025entropic}, some of which change $\net$ by reweighting positive and negative samples.
On-policy, entropy follows the covariance between advantages and log-probabilities~\citep{cui2025entropy,wang2026entropydyn} and responds to clipping~\citep{park2025clip}, log-probability-dependent advantages~\citep{petrenko2026repo}, and positive-only reinforcement~\citep{zhu2025nsr}.
Our gap is the sampler expectation of the entropy polarity of \citet{zhang2026polarity}.
SCOPE-RL~\citep{scoperl2025} finds that positive samples raise entropy under high-temperature sampling, which the sign law predicts for a flatter sampler.

\paragraph{Baselines and mismatch.}
Without importance correction, a constant baseline adds a $\KL(\smp\|\pol)$ gradient signed by the mean advantage~\citep{leroux2016efficient,leroux2025topr}, acts as pessimism~\citep{cao2026anatomy}, and matters even on-policy~\citep{chung2021beyond}.
A high baseline makes the limit of a full-support behavior policy deterministic~\citep{arnal2025asymre}, but under truncation it raises entropy (\cref{thm:sign}).
\citet{marek2026score} formalize the drift for several kinds of mismatch, describe the pull toward the sampler after promising prefixes and the push after flawed ones, propose score centering, and note that group centering leaves the drift.
\Cref{prop:aggregate} identifies what remains, and \citet{han2026baselines} give a gradient-level split of the same kind.
\GW{} has the form of variance-optimal baselines~\citep{greensmith2004variance,hao2025opo,li2026otb}, and balanced aggregation removes the sign--length coupling behind $\net$~\citep{liu2025drgrpo,balanced2026}.
Drift can compound~\citep{marek2026score,modal2026mixed}, and cooler sampling without correction collapses~\citep{pan2026cool}; \cref{prop:feedback} shows that the compounding is asymmetric.

%% file: sections/conclusion.tex
\section{Conclusion}
\label{sec:conclusion}

Under a sampler sharper or flatter than the policy, a constant advantage shift moves entropy with a sign fixed by the sampler.
In group-based RLVR this shift is the token-weighted net advantage, which masking and failing push to opposite signs.
Because the remaining drift feeds back asymmetrically, zeroing the net advantage only delays an entropy explosion, which learner-side renormalization or per-step cancellation prevents in our 100-step runs.
The analysis is first order and leaves the size of the response to measurement.

%% file: sections/appendix.tex
\section{Proofs and Examples}
\label{app:proofs}

Throughout, $\pol$ is a distribution with full support on a finite vocabulary $\mathcal V$, $H=-\sum_v\pol_v\log\pol_v$, $u_v=\pol_v(\log\pol_v+H)$, and for a set $X\subseteq\mathcal V$ we write $m_X=\E_\pol[u\mid X]$.

\subsection{Entropy change at one position (\cref{eq:phi})}
With $\pol=\mathrm{softmax}(z)$, $\partial\pol_w/\partial z_v=\pol_w(\mathbf 1[w=v]-\pol_v)$, so
\[
\frac{\partial H}{\partial z_v}=-\sum_w(\log\pol_w+1)\,\pol_w(\mathbf 1[w=v]-\pol_v)=-\pol_v(\log\pol_v+1)+\pol_v(1-H)=-u_v.
\]
The gradient of $A\log\pol_a$ with respect to $z$ is $A(e_a-\pol)$, so a step of size $\eta$ changes entropy by $-\eta A\sum_v u_v(\mathbf 1[v=a]-\pol_v)+O(\eta^2)=-\eta A\,(u_a-\E_\pol[u])+O(\eta^2)$.
Note that $\E_\pol[u]=\sum_v\pol_v^2(\log\pol_v+H)=\Cov_\pol(\pol_v,\log\pol_v)\ge0$, since both arguments are increasing functions of $\pol_v$.

\subsection{Proof of \cref{thm:sign}}

Call $\keep\subseteq\mathcal V$ an \emph{upper set} if it contains every token whose probability exceeds $\min_{v\in\keep}\pol_v$, and a \emph{lower set} if its complement is an upper set.
Ties may be split arbitrarily.
For nonempty $\keep$ let $q_\keep=\pol\mathbf 1_\keep/\pol(\keep)$.

\begin{lemma}
\label{lem:trunc}
If $\pol$ is non-uniform and $\keep$ is an upper set with $0<\pol(\keep)<1$, then $\gap(q_\keep)>0$.
\end{lemma}

\begin{proof}
Since $\E_\pol[u]=\pol(\keep)m_\keep+\pol(\keep^c)m_{\keep^c}$, we have $\gap(q_\keep)=m_\keep-\E_\pol[u]=\pol(\keep^c)(m_\keep-m_{\keep^c})$, so it suffices to show $m_\keep>m_{\keep^c}$.
Let $U=\{v:\pol_v\ge e^{-H}\}$, the set where $u_v\ge0$, and let $m=\min_{v\in\keep}\pol_v$.
The sets $\keep$ and $U$ are nested: if $m\ge e^{-H}$ then $\keep\subseteq U$, and otherwise every $v\in U$ has $\pol_v\ge e^{-H}>m$ and hence $v\in\keep$.
Because $\pol$ is non-uniform, its least probable token $v_{\min}$ satisfies $\pol_{v_{\min}}<e^{-H}=\exp(\E_\pol[\log\pol])$, a weighted geometric mean, so $u_{v_{\min}}<0$.

\emph{Case $\keep\subseteq U$.}
Take $v\notin\keep$ and $w\in\keep$.
If $v\notin U$, then $u_v<0\le u_w$.
If $v\in U$, then $e^{-H}\le\pol_v\le m\le\pol_w$; the function $p\mapsto p(\log p+H)$ has derivative $\log p+1+H>0$ on $[e^{-H},1]$, so $u_v\le u_w$.
Hence $\max_{\keep^c}u\le\min_\keep u$, and since $v_{\min}\in\keep^c$ has $u_{v_{\min}}<0\le\min_\keep u$, the $\pol$-weighted mean satisfies $m_{\keep^c}<\min_\keep u\le m_\keep$.

\emph{Case $U\subseteq\keep$.}
Every $v\in\keep^c$ has $u_v<0$, so $m_{\keep^c}<0$ and $\sum_{v\in\keep}\pol_vu_v=\E_\pol[u]-\sum_{v\in\keep^c}\pol_vu_v>\E_\pol[u]\ge0$.
Hence $m_\keep>0>m_{\keep^c}$.
\end{proof}

\begin{proof}[Proof of \cref{thm:sign}]
Let $\rho_v=\smp_v/\pol_v$, so $\sum_v\pol_v\rho_v=1$.
By the layer-cake identity $\rho_v=\int_0^\infty\mathbf 1[\rho_v\ge s]\,ds$,
\[
\smp=\int_0^{\infty}\pol(\keep_s)\,q_{\keep_s}\,ds,\qquad \keep_s=\{v:\rho_v\ge s\},\qquad \int_0^\infty\pol(\keep_s)\,ds=1.
\]
Since $\gap(\smp)=\E_\smp[u]-\E_\pol[u]$ is linear in $\smp$ and the weights integrate to one, $\gap(\smp)=\int_0^\infty\pol(\keep_s)\,\gap(q_{\keep_s})\,ds$, where terms with $\keep_s=\emptyset$ have zero weight.
If $\smp$ is sharper than $\pol$, then $v\in\keep_s$ and $\pol_w>\pol_v$ imply $\rho_w\ge\rho_v\ge s$, so every $\keep_s$ is an upper set.
For $s\le\min_v\rho_v$, $\keep_s=\mathcal V$ and $\gap(q_{\keep_s})=0$.
For $s\in(\min_v\rho_v,\max_v\rho_v]$, $\keep_s$ is a nonempty upper set with $\pol(\keep_s)<1$, so $\gap(q_{\keep_s})>0$ by \cref{lem:trunc}.
Because $\smp\neq\pol$, $\rho$ is not constant and this interval has positive length; hence $\gap(\smp)>0$.
If $\smp$ is flatter than $\pol$, the same argument shows that every $\keep_s$ is a lower set.
For a nonempty lower set $R$ with $\pol(R)<1$, $R^c$ is an upper set and \cref{lem:trunc} gives $m_{R^c}>m_R$, so $\gap(q_R)=\pol(R^c)(m_R-m_{R^c})<0$; integrating gives $\gap(\smp)<0$.
\end{proof}

\paragraph{Which samplers are sharper.}
Top-$k$ keeps the $k$ most probable tokens; top-$p$ keeps the shortest prefix of the tokens sorted by probability whose mass reaches $p$; min-$p$ keeps tokens with $\pol_v\ge p_{\min}\max_w\pol_w$.
Each keeps an upper set, and applying them in sequence keeps an upper set of $\pol$ because renormalizing within a set preserves the order of probabilities.
For any truncation, $\rho=\mathbf 1_\keep/\pol(\keep)$, which satisfies the definition of sharper.
Temperature gives $\rho\propto\pol^{1/T-1}$, increasing for $T<1$ and decreasing for $T>1$.
Temperature followed by truncation gives $\rho\propto\pol^{1/T-1}\mathbf 1_\keep$ with $\keep$ an upper set, which is sharper for $T\le1$; for $T>1$ it rises and then drops to zero, so the sampler is neither sharper nor flatter and the sign of $\gap$ depends on $\pol$.
Mixing with the uniform distribution gives $\rho_v=(1-\epsilon)+\epsilon/(|\mathcal V|\pol_v)$, which is decreasing, hence flatter.

\subsection{Proof of \cref{prop:aggregate}}
In the softmax parameterization each distinct prefix $s$ has its own logits.
The surrogate gradient changes them by $\Delta z(s)=\eta\sum_{(i,t):s_{i,t}=s}\omega_iA_i(e_{y_{i,t}}-\pol(\cdot\mid s))$, and by \cref{eq:phi} the first-order change of $\sum_sH(s)$ over distinct visited prefixes is $\Delta H_{\mathrm{tot}}=-\eta\sum_i\omega_iA_i\sum_t\Phi_{i,t}(y_{i,t})$, where $\Phi_{i,t}$ is \cref{eq:phi} at prefix $s_{i,t}$.
Write $\Phi_{i,t}(y_{i,t})=\gap_{i,t}+(\Phi_{i,t}(y_{i,t})-\gap_{i,t})$.
The second part gives the on-policy term.
For the first part, within group $g$,
\[
\sum_{i\in g}\omega_iA_i\sum_t\gap_{i,t}=\sum_{i\in g}\omega_iL_iA_i\bar\gap_i=W_g\big(\net_g\bar\gap_g+\Cov_g(A,\bar\gap)\big),
\]
where the mean and covariance are weighted by $\omega_iL_i/W_g$.
Adding a constant $c$ to the advantages of group $g$ increases $\net_g$ by $c$, leaves $\Cov_g(A,\bar\gap)$ unchanged, and changes the on-policy term by $-\eta c\sum_{i\in g}\omega_i\sum_t(\Phi_{i,t}(y_{i,t})-\gap_{i,t})$, whose terms have zero mean given their prefixes because $y_{i,t}\sim\smp(\cdot\mid s_{i,t})$ and $\omega_i$ is fixed.
Under exact sampling $\gap_{i,t}=0$ at every prefix.

\paragraph{Example: the gap covariance survives $\net=0$.}
Let a response first choose one of two branches uniformly and then emit one more token; the reward is 1 on the first branch and 0 on the second, so all responses have equal length and every group has $\net_g=0$ under a group-mean baseline.
Let the policies at the two second positions be $(0.8,0.2)$ and $(0.99,0.01)$, both sampled with top-1 truncation.
With $G=8$, the expected first-order entropy change is $-0.0096\,\eta$: positive advantages land on the prefix with the larger gap, so the gap covariance is nonzero although \NZ{} leaves the advantages unchanged.
Score centering or support renormalization removes it.

\paragraph{Example: length-dependent weights bias exact sampling.}
Let the first token be drawn from $(0.8,0.2)$ without truncation; the first outcome ends the response ($L=1$) and the second adds one uniformly random token ($L=2$).
With $A\equiv c$, $G=8$, and token weights $\omega=1/\sum_iL_i$, the net advantage is the constant $c$ and there is no mismatch, yet the expected first-order entropy change is $-0.0060\,\eta c$.
The weight is larger in groups where more responses are short, so the update favors the short outcome.
This is the length bias of token-level aggregation~\citep{liu2025drgrpo}; it is of order $1/G$ and odd in $c$.

\subsection{Identities in \cref{tab:contracts}}
\begin{table}[htbp]
\caption{Net advantage of one group (uniform token weights except in the sequence-mean rows) and the direction of the net-advantage drift under a sharper sampler in our runs. Under \MASK{}, masked responses enter the baseline with zero reward; $M$ is the set of unmasked responses, $n_{\mathrm{mask}}=G-|M|$, and $\Cov_M$, $\bar L_M$, $\bar r_M$ are taken over $M$.}
\label{tab:contracts}
\vskip 0.05in
\centering
\footnotesize
\setlength{\tabcolsep}{3pt}
\begin{tabular}{@{}lll@{}}
\toprule
Rule & $\net_g$ & Drift \\
\midrule
Group mean & $\Cov_g(r,L)/\bar L_g$ & by sign \\
\MASK{} & $\tfrac{\Cov_M}{\bar L_M}+\tfrac{n_{\mathrm{mask}}}{G}\bar r_M$ & $H\!\downarrow$ \\
\FAIL{}, \PEN{} & $\Cov_g/\bar L_g$ & $H\!\uparrow$ \\
Length-weighted (\NZ) & $0$ & none \\
Seq.-mean agg., \FAIL{} & $0$ & none \\
Seq.-mean agg., \MASK{} & $\tfrac{n_{\mathrm{mask}}}{G}\bar r_M$ & $H\!\downarrow$ \\
\bottomrule
\end{tabular}
\vskip -0.1in
\end{table}

Let all responses in group $g$ share the token weight $\omega_g$, and let $\Cov_g$ and $\bar L_g$ be the population covariance and mean over the group.
\begin{itemize}
\item \emph{Group-mean baseline.} $\net_g=\sum_iL_i(r_i-\bar r)/\sum_iL_i=(\frac1G\sum_iL_ir_i-\bar L_g\bar r)/\bar L_g=\Cov_g(r,L)/\bar L_g$. Dividing advantages by $\sigma_g$ divides $\net_g$ by $\sigma_g$. The batch-level $\net$ is the average of $\net_g$ with weights $W_g$, which are equal across groups under prompt-mean aggregation.
\item \emph{\MASK.} Masked responses have $r_i=0$ and enter only the baseline $\bar r_G=\frac1G\sum_{i\in M}r_i=\frac{|M|}{G}\bar r_M$, where $M$ is the set of trained responses. Then $\net_g=\sum_{i\in M}L_i(r_i-\bar r_G)/\sum_{i\in M}L_i=\Cov_M(r,L)/\bar L_M+\bar r_M-\bar r_G=\Cov_M(r,L)/\bar L_M+\frac{n_{\mathrm{mask}}}{G}\bar r_M$. The second term is nonnegative when rewards are nonnegative, as in all our runs. The first term can be negative, so $\net_g$ itself need not be positive.
\item \emph{\NZ.} With $b_g=\sum_jL_jr_j/\sum_jL_j$, $\net_g=\sum_iL_i(r_i-b_g)/\sum_iL_i=0$.
\item \emph{Sequence-mean-token-mean.} With $\omega_i\propto1/L_i$, $\net_g$ is the plain mean of the trained responses' advantages, which is zero when no responses are masked because group-mean advantages sum to zero within each group; under \MASK{} it is $\bar r_M-\bar r_G=\frac{n_{\mathrm{mask}}}{G}\bar r_M$.
\end{itemize}

\subsection{Proof of \cref{prop:feedback}}
Write $p_h=(1-x)/k$ and $p_t=x/(V-k)$ for the head and tail probabilities, so $\Lambda=\log(p_h/p_t)$.
Under top-$k$ the sampler is $\smp=\pol\mathbf 1_\keep/(1-x)$, so $\smp-\pol$ equals $p_hx/(1-x)=x/k$ on every head token and $-p_t=-x/(V-k)$ on every tail token.
The update therefore keeps the head logits equal and the tail logits equal, and $\dot\Lambda=\eta c\kappa x$.
Since $x=(V-k)/(ke^\Lambda+V-k)$, $\mathrm dx/\mathrm d\Lambda=-x(1-x)$ and $\dot x=-\eta c\kappa x^2(1-x)$.
For the gap, $H=-(1-x)\log p_h-x\log p_t$ gives $\log p_h+H=x\Lambda$ and $\log p_t+H=-(1-x)\Lambda$, so $u=p_hx\Lambda$ on the head and $u=-p_t(1-x)\Lambda$ on the tail.
The truncation formula $\gap=x\,(\E_\pol[u\mid\keep]-\E_\pol[u\mid\keep^c])$ then gives $\gap=x\big(p_hx\Lambda+p_t(1-x)\Lambda\big)=\kappa\Lambda x^2(1-x)$.
As a check, $\mathrm dH/\mathrm dx=\Lambda$, so $\dot H=\Lambda\dot x=-\eta c\,\gap$, as \cref{eq:gap} requires.
For (ii), $\int\mathrm dx/(x^2(1-x))=-1/x+\log(x/(1-x))$; with $c<0$, integrating from $x_0$ to $1/2$ gives $t_{1/2}$.
With $c>0$, the same integral gives $1/x(t)-\log\frac{x(t)}{1-x(t)}=1/x_0-\log\frac{x_0}{1-x_0}+\eta c\kappa t$, so $1/x(t)=\Theta(t)$.
For (iii), $\dot\Lambda=\rho-\eta|c|\kappa x$ has the sign of $x^\ast-x$ with $x^\ast=\rho/(\eta|c|\kappa)$; since $x$ decreases in $\Lambda$, $x>x^\ast$ lowers $\Lambda$ and raises $x$ further, and $x<x^\ast$ does the reverse.

\begin{figure}[htbp]
\centering
\includegraphics[width=0.62\linewidth]{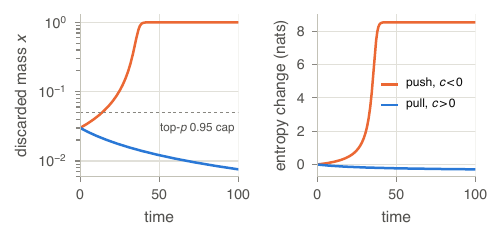}
\caption{\Cref{prop:feedback} in the two-level model ($V=151{,}000$, $k=20$, $x_0=0.03$): discarded mass and entropy change under a push ($c<0$) and a pull ($c>0$) of the same size, with time in units of $1/(\eta|c|\kappa k)$. The push reaches the uniform policy in finite time; the pull slows down. Dashed: the cap on the discarded mass under top-$p$ 0.95.}
\label{fig:feedback}
\end{figure}

\paragraph{Two classes of prefixes.}
Take one two-level prefix of each of two classes, visited at rates $n_a$ and $n_b$ and receiving advantages $-c_a<0$ and $+c_b>0$, so that $\dot\Lambda_j=\eta n_jA_j\kappa x_j$ and $\dot H_j=-\eta n_jA_j\gap_j$, and let total entropy be $H_a+H_b$.
By (ii), $x_a$ reaches $1/2$ by time $t_{1/2}(x_{a,0},n_ac_a)$ whatever $c_b$ is, while $H_b$ cannot fall below $\log k$.
Total entropy therefore exceeds its initial value once $H_a(1/2)-H_a(0)>H_b(0)-\log k$, even if $n_bc_b\gap_b>n_ac_a\gap_a$ at $t=0$ and entropy first falls; this holds when the pulled class starts close to its head, as confident prefixes do (for $V=151{,}000$, $k=20$, $x_{a,0}=0.03$, and $x_{b,0}=0.02$, the two sides are 4.8 and 0.3 nats).
If entropy is instead averaged with weights $n_a$ and $n_b$, the condition becomes $n_a(H_a(1/2)-H_a(0))>n_b(H_b(0)-\log k)$, which fails when the pushed class is rare enough.
Adding a constant $s<c_a$ to both advantages replaces $c_a$ by $c_a-s$ in $t_{1/2}$, which delays but does not prevent the takeoff.
A baseline that sets $n_aA_a\gap_a+n_bA_b\gap_b=0$ at every instant, which is \GW{} for one group containing both classes and also the per-batch shift of \cref{sec:exp-long}, gives $\dot H_a+\dot H_b=0$.
Under top-$p$ instead of top-$k$, $\|\smp-\pol\|_1=2\pol(\keep^c)\le2(1-p)$ for every policy, so the logit drift is bounded uniformly in the state.

\subsection{Proof of \cref{prop:fixes}}
(a) follows from \cref{eq:aggregate} with every $\net_g=0$.
For (b), the proof of \cref{prop:aggregate} gives $D_{\mathrm{net}}+D_{\mathrm{cov}}=-\eta\sum_g\sum_{i\in g}\omega_iL_iA_i\bar\gap_i$; with $A_i=r_i-b_g$ and $b_g$ as stated, each group's sum is zero. By \cref{thm:sign} every $\gap_{i,t}\ge0$ under a sharper sampler (zero where $\smp=\pol$ at that prefix), so the weights $\omega_iL_i\bar\gap_i$ are nonnegative.
For (c), $\nabla\log\pol(\keep)=\sum_{v\in\keep}\nabla\pol_v/\pol(\keep)=\sum_{v\in\keep}q_v\nabla\log\pol_v=\E_{v\sim\smp}[\nabla\log\pol(v)]$ because $\smp=q_\keep$; hence $\E_{y\sim\smp}[\nabla\log\pol_\keep(y)]=0$.
Score centering replaces $\nabla\log\pol(y)$ by $\nabla\log\pol(y)-\E_{v\sim\smp}[\nabla\log\pol(v)]$~\citep{marek2026score}, which is the same vector.
In logit space the update becomes $\eta A(e_y-\smp)$, so the first-order entropy change at the prefix is $-\eta A(u_y-\E_\smp[u])$, whose mean under $\smp$ is zero; both mismatch terms of \cref{eq:aggregate} therefore vanish.
For temperature mismatch, the analogous statement holds for a learner that computes $\log\mathrm{softmax}(z/T)$, whose distribution then equals $\smp$.
For (d), $\E_{y\sim\smp}[w(y)c\nabla\log\pol(y)]=c\sum_{y\in\keep}\smp_yw_y\nabla\log\pol_y$.
With $w=\pol/\smp=\pol(\keep)$ on $\keep$ this is $c\sum_{y\in\keep}\nabla\pol_y=c\nabla\pol(\keep)=\pol(\keep)\,c\nabla\log\pol(\keep)=\pol(\keep)\E_\smp[c\nabla\log\pol]$.
In the softmax parameterization the expected logit change is $c\sum_{y\in\keep}a_y(e_y-\pol)$ with $a_y=\smp_yw_y$.
If it vanished, its coordinate at any $v\notin\keep$ (which exists since $\pol(\keep)<1$) would give $\pol_v\sum_ya_y=0$, so $\sum_ya_y=0$, and then each coordinate $y\in\keep$ would give $a_y=0$; hence only $w\equiv0$ makes the expected logit update vanish.
The statement concerns the logit update: weights that depend on $\Phi(y)$ could cancel the first-order entropy change at one prefix while leaving the update biased, but importance ratios are not of this form.

\paragraph{Example: sequence-level ratios can reverse the drift.}
Let the first token be drawn from $(0.8,0.2)$ without truncation and the second from $(0.51,0.49)$ after the first outcome and from $(0.99,0.01)$ after the second, both with top-1 truncation; let $A\equiv1$ with fixed weights.
The expected first-order entropy change is $-0.0080\,\eta$ without importance ratios and $-0.0042\,\eta$ with exact token-level ratios, which rescale it.
With the exact sequence-level ratio, the product of the kept masses (0.51 or 0.99) weights the first token as well and favors the branch with the larger kept mass, and the change becomes $+0.0299\,\eta$.

\section{Numerical Validation of the Theory}
\label{app:numerics}

\paragraph{Sign law.}
We evaluated $\gap$ on 14{,}400 random distributions and on 4{,}920 adversarially constructed ones (a single dominant token, a plateau of medium-probability tokens, a long tail, and many ties), for every prefix truncation, temperatures $T\in\{0.3,0.7,0.95,1.05,1.5,3\}$, and uniform mixtures, and additionally minimized $\gap$ numerically over sharper samplers.
No sharper sampler gave $\gap<0$; the smallest value, $-8\times10^{-16}$, is rounding error.

\paragraph{Exact single-position computation.}
We take $|\mathcal V|=1000$, give the top token probability $\pol_1=0.9$, and spread the remaining mass over the tokens of rank $k=2,\dots,1000$ proportionally to $k^{-1.1}$.
Enumerating the expected one-step entropy change exactly (no sampling noise), with $A\equiv c$ and $\eta=0.1$, the odd part of $\E[\Delta H]/\eta$ equals $-c\,\gap$ to within $10^{-4}$: $\gap=0.0595$ for top-$k$ 20 followed by top-$p$ 0.95, and $0.0314$ for top-$p$ 0.95; under exact sampling the odd part is $2\times10^{-5}$, and only an even $O(\eta)$ term remains.
With the same tail, single-token supports give small gaps that vanish as the top token dominates: $\gap=0.015$ at $\pol_1=0.957$ and $0.002$ at $\pol_1=0.985$.

\paragraph{Tabular RLVR simulation.}
We built a synthetic RLVR task with 8 problems and a tabular autoregressive policy whose context is (problem, position, running sum mod 5, a derailment flag); responses are correct if they stop with an allowed residue, and a derailing token leads to an over-length response, so over-length responses are long and low-confidence, as in our language-model data.
With groups of 8, a group-mean baseline, prompt-mean aggregation, and top-$p$ 0.9 sampling, the first-order decomposition predicts the realized entropy change over 300 SGD steps to within 3\%, and the mismatch drift has the sign of $-\net$.
Under Adam, the dose-response slope of entropy change on $\net$ is $-0.41$ under truncated sampling and $-0.04$ under exact sampling (5 seeds).

\section{Experimental Details}
\label{app:details}

\begin{table}[htbp]
\caption{Training and evaluation configuration.}
\label{tab:config}
\centering
\footnotesize
\begin{tabular}{@{}p{0.2\textwidth}p{0.3\textwidth}p{0.44\textwidth}@{}}
\toprule
 & 2B & 9B \\
\midrule
Model & Qwen3.5-2B (non-thinking) & MiMo-V2.6-Distill-Qwen-9B \\
Start & released checkpoint (kicked runs: logits divided by 1.2) & step 20 of an earlier run with the same trainer (weights and Adam state) \\
Groups per step $\times$ group size & $16\times8$ & $32\times8$ \\
Steps & 15 (temperature runs: 10--15; long runs: 100; kicked: 50) & 10 (step 20 $\to$ 30) \\
Generation limit & 2{,}048 & 8{,}192 \\
Reward & binary & binary with \PEN{} ($r=\mathbf 1[\text{correct}]-\max(0,(L-6000)/2192)$ for completed responses) \\
Optimizer & \multicolumn{2}{p{0.74\textwidth}}{AdamW, lr $10^{-6}$, $\beta=(0.9,0.999)$, weight decay 0.01, gradient clipping 1.0} \\
Objective & \multicolumn{2}{p{0.74\textwidth}}{dual-clip PPO (0.2, 3.0), one step per batch, ratio $=1$ (one control: four minibatch steps), no KL or entropy term, prompt-mean aggregation} \\
Seeds & 3 per cell; $T{=}0.7$ and new controls: 2; long runs: 3 (renormalized, fixed shifts, top-$p$, exact: 2) & 4 per cell; \MASK{}: 1 \\
Entropy probe & 64 held-out prompts $\times1$ response & 64 held-out prompts $\times1$ response \\
\bottomrule
\end{tabular}
\end{table}

\paragraph{Data.}
The pool contains 7{,}496 math problems with verifiable answers.
We hold out 750 problems from training; the held-out evaluation set is a fixed random subset of 300 of them, and the 64 probe prompts are drawn from that subset.
The fixed evaluation set contains 32 problems, each sampled 8 times; 3 of them are also probe prompts.

\paragraph{Samplers.}
Rollouts are generated by vLLM.
Truncated sampling applies top-$k$ 20 and then top-$p$ 0.95 at $T=1$.
Support renormalization recomputes the kept set from the learner's logits with the same rule and renormalizes the learner's log-probability of the sampled token over the kept set joined with that token; the sampled token falls outside the recomputed set for about 0.06\% of tokens, because of numerical differences between engine and learner.
For $T=0.7$ and $T=1.2$, no truncation is applied and the learner computes log-probabilities at $T=1$.

\paragraph{Probe set and entropy.}
The probe set contains one response per prompt for 64 held-out prompts, sampled once from the starting policy with truncated sampling.
Before every update we compute, by teacher forcing, the full-softmax entropy at every response token and average over tokens.
Pure-shift derivatives on a probe set sampled with exact sampling agree with those on the truncated probe set (\cref{app:more}).

\paragraph{Grading and statistics.}
Answers are extracted from the final \texttt{\textbackslash boxed\{\}} and graded with math-verify.
All 9B evaluations, including those of runs trained with exact sampling, use truncated sampling so that arms are compared under the same evaluation protocol.
For accuracy we treat the training seed as the unit: we compute the paired difference between rules for each seed (arms with the same seed share prompts and sampling seeds) and report the mean, a $t$-based 95\% interval with 3 degrees of freedom, and a two-sided paired $t$-test.
For the dose-response we fit the slope separately for each seed and report the mean and standard error over seeds.
The batch net advantage $\net$ weights groups by $W_g$ (\cref{sec:aggregate}), and a run's $\net$ is its mean over training steps; temperature runs last 10--15 steps.

\paragraph{Pure-shift probes.}
The directional derivative is $\langle\nabla_\theta H_{\mathrm{probe}},\,\sum_i\omega_i\sum_t\nabla_\theta\log\pol_\theta(y_{i,t}\mid s_{i,t})\rangle$ evaluated on a training batch of 32 groups split into four sub-batches.
The Adam version replaces the gradient by the update direction $D\odot g$, where $D$ is the diagonal preconditioner from the second-moment estimate stored at the fork, scaled to one step at learning rate $10^{-6}$.
The odd-part experiment trains two chains with $A\equiv+1$ and $A\equiv-1$ on identical prompts and sampling seeds, so their first updates are exactly opposite.

\paragraph{Compute.}
2B experiments ran on 8$\times$H100 nodes with one run per GPU; 9B experiments ran on single H200 and 8$\times$H100 nodes.

\section{Additional Results}
\label{app:more}

\begin{table}[htbp]
\caption{Mismatch drift on three 9B training batches sampled with top-$k$ 20 and top-$p$ 0.95 at the fork point, split as in \cref{prop:aggregate} into the net-advantage part and the gap covariance, computed from the measured gap at every sampled prefix with prompt-mean weights. Columns 4--5 are the sums in \cref{eq:aggregate}; the entropy change is $-\eta$ times them, so negative entries raise entropy. The same rollouts are scored under each rule.}
\label{tab:decomp}
\centering
\small
\begin{tabular}{@{}llrrrr@{}}
\toprule
Batch & Rule & Mean $\net_g$ & $\sum_gW_g\net_g\bar\gap_g$ & $\sum_gW_g\Cov_g(A,\bar\gap)$ & Net share \\
\midrule
1 & \FAIL{} & $-0.077$ & $-3.6\times10^{-4}$ & $-1.2\times10^{-4}$ & 76\% \\
 & \PEN{} & $-0.079$ & $-3.7\times10^{-4}$ & $-1.1\times10^{-4}$ & 76\% \\
 & \MASK{} & $+0.164$ & $+6.9\times10^{-4}$ & $-1.1\times10^{-4}$ & 119\% \\
2 & \FAIL{} & $-0.054$ & $-2.4\times10^{-4}$ & $-0.7\times10^{-4}$ & 78\% \\
 & \PEN{} & $-0.058$ & $-2.5\times10^{-4}$ & $-0.7\times10^{-4}$ & 78\% \\
 & \MASK{} & $+0.119$ & $+6.1\times10^{-4}$ & $-0.6\times10^{-4}$ & 110\% \\
3 & \FAIL{} & $-0.022$ & $-1.1\times10^{-4}$ & $-0.6\times10^{-4}$ & 66\% \\
 & \PEN{} & $-0.043$ & $-1.9\times10^{-4}$ & $-0.8\times10^{-4}$ & 71\% \\
 & \MASK{} & $+0.119$ & $+4.7\times10^{-4}$ & $-0.5\times10^{-4}$ & 112\% \\
\bottomrule
\end{tabular}
\end{table}

\begin{table}[htbp]
\caption{Probes of a pure advantage shift. Derivative: first-order change of probe-set entropy along the update of $A\equiv1$ on one training batch, in nats per unit SGD learning rate, or in $10^{-3}$ nats per Adam step at learning rate $10^{-6}$ with the optimizer state of the fork (mean $\pm$ s.e.\ over four sub-batches). Odd part: $\frac12[H(+1)-H(-1)]$ after five steps with $A\equiv\pm1$.}
\label{tab:probes}
\vskip 0.05in
\centering
\footnotesize
\setlength{\tabcolsep}{3pt}
\begin{tabular}{@{}llrrr@{}}
\toprule
Probe & Model & Truncated & Exact & Ratio \\
\midrule
SGD deriv. & 9B & $-10.3\pm1.1$ & $-0.41\pm0.14$ & 4\% \\
Adam deriv. & 9B & $-125\pm13$ & $-4.7\pm1.9$ & 4\% \\
SGD deriv. & 2B & $-10.7\pm0.4$ & $-0.03\pm0.30$ & ${<}6\%$ \\
Odd part & 2B & $-0.079$ & $-0.0010$ & 1.3\% \\
\bottomrule
\end{tabular}
\vskip -0.1in
\end{table}

\paragraph{Pure-shift probes and the gap-covariance intervention.}
\Cref{tab:probes} gives the pure-shift probes of \cref{sec:exp-exists}; the first-order change of one Adam step at learning rate $10^{-6}$ equals that of an SGD step with learning rate about 0.012, consistent with the learning rate of 0.007 that matches Adam's 15-step slope in the SGD runs below, and the truncated-to-exact ratio is 4\% under both optimizers; over five steps with $A\equiv\pm1$ the odd part is $-0.079$ and the even part $+0.008$ under truncation, against $-0.0010$ and $-0.0004$ under exact sampling.
For the intervention of \cref{sec:exp-predicts} we keep every $\net_g=0$ and add $\kappa\zeta_i$ with $\kappa=0.3$, where $\zeta_i$ is the response's mean log-probability standardized within its group with zero length-weighted mean, so $\net$ is unchanged; the interaction is about 1.8 times the amount predicted from the logged change in $D_{\mathrm{cov}}$.
The truncated dose-response line crosses zero at $\net=+0.025$, and \MASK{} ($\net=+0.30$, outside the tested range) lowers entropy less than the line implies, because its $D_{\mathrm{cov}}$ points against its $D_{\mathrm{net}}$ and masking removes over-length responses.

\begin{table}[htbp]
\caption{Entropy slope ($10^{-3}$ nats/step) of 2B runs under three advantage rules and three samplers (mean $\pm$ s.e.\ over 3 seeds). Rel.: paired \MASK$-$\FAIL{} difference relative to its value under truncated sampling. Under truncation the arms do not differ in batch accuracy (0.56--0.59) or over-length rate (0.27--0.31).}
\label{tab:e4}
\vskip 0.05in
\centering
\footnotesize
\setlength{\tabcolsep}{3pt}
\begin{tabular}{@{}lrrrr@{}}
\toprule
Sampler & \MASK{} & \FAIL+\NZ{} & \FAIL{} & Rel. \\
\midrule
Mean $\net$ & $+0.30$ & $0$ & $-0.08$ & \\
\midrule
Truncated & $-9.4\pm0.2$ & $+2.4\pm0.4$ & $+11.9\pm0.6$ & 100\% \\
Exact & $-0.4\pm0.5$ & $-1.0\pm0.3$ & $-1.1\pm0.2$ & $-3\%$ \\
Renormalized & $+0.1\pm0.1$ & $-0.3\pm0.1$ & $-0.7\pm0.3$ & $-4\%$ \\
\bottomrule
\end{tabular}
\vskip -0.1in
\end{table}

\paragraph{Logged drift and entropy slope across rules.}
\Cref{fig:drift} plots the mean logged first-order drift over the first 15 steps against the 15-step entropy slope for every 2B run under the truncated sampler with logged gaps (35 runs, 13 rules): \MASK{}, \FAIL{}, and \FAIL+\NZ{} (first 15 steps of the long runs, and \FAIL+\NZ{} also from the 15-step runs), \FAIL+\NZ{} with the covariance intervention and with the fixed shift of $+0.05$, \FAIL+\GW{} (15- and 100-step runs), the surprisal-weighted baseline, \MASK{} and \FAIL{} under sequence-mean aggregation, and \MASK{} and \FAIL{} with exact and with clipped importance ratios.
A single line fits all of them ($r=0.97$; slope $0.198\times10^{-3}$ nats per step per $10^{-5}$ of drift, intercept $+0.14$).
Fitted without \MASK{}, the line is steeper ($0.264$, intercept $-1.28$) and predicts $-15.6$ for \MASK{} against the measured $-9.7$, still much closer than the net-advantage line ($-29$).
The six \MASK{} runs with shifted advantages, run after the fit, fall within $4\times10^{-3}$ of the line.
The response flattens at large negative drift; under SGD it does not (next paragraph), so this curvature comes from Adam.
The 12 runs with top-$p$ 0.99 and 0.8 are excluded: across samplers the response is concave (\cref{sec:exp-sign}), and a line fitted to them underpredicts the truncated runs by a factor of two.

\begin{figure}[t]
\centering
\includegraphics[width=0.62\linewidth]{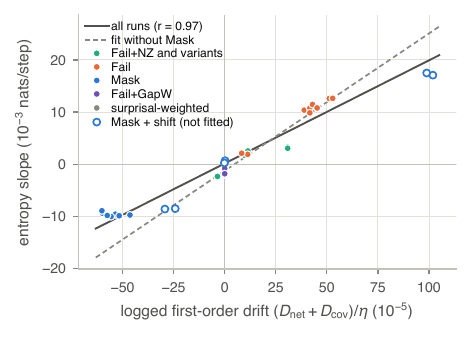}
\caption{Logged first-order mismatch drift (mean over the first 15 steps) against the 15-step entropy slope, 35 runs of 2B under top-$k$ 20 and top-$p$ 0.95. Solid: fit to all runs; dashed: fit without \MASK{}. Open markers: \MASK{} with advantages shifted by $-\lambda\net_g$, $\lambda\in\{0.5,1,1.5\}$, run after the fit.}
\label{fig:drift}
\end{figure}

\paragraph{How much of the mismatch drift the net advantage carries.}
\Cref{tab:decomp} evaluates the two mismatch terms of \cref{prop:aggregate} on real 9B batches, using the measured gap at every sampled prefix.
The gap covariance is negative under every rule, because failed responses visit prefixes with larger gaps, so it raises entropy.
Under \FAIL{} and \PEN{} it adds to the net-advantage drift, which carries 66--78\% of the total; under \MASK{} it offsets 9--16\% of the net-advantage drift.

\paragraph{From visited prefixes to held-out prompts.}
\Cref{prop:aggregate} describes entropy at the visited prefixes, whereas we measure it on held-out probe prompts.
On one 9B training batch, the ratio of the first-order effect of the \FAIL+\PEN{} rule to that of a pure shift is $-0.101$ when predicted from the per-response gaps and $-0.081$ when measured on probe-set entropy; for a counterfactual \MASK{} rule on the same samples the two ratios are $0.196$ and $0.178$.

\paragraph{Same samples, different rules.}
On the same batch, the first-order change of probe entropy under \FAIL+\PEN{} is $(+0.87\pm0.27)\times10^{-3}$ per unit SGD step, against $(-0.005\pm0.12)\times10^{-3}$ under \FAIL+\PEN+\NZ{}, which differs from it only by group-wise shifts; the counterfactual \MASK{} rule ($\net=+0.19$) gives $(-1.91\pm0.63)\times10^{-3}$.

\paragraph{Pure-shift directional derivatives on two probe sets.}
For the 9B model, the derivative under truncated training samples is $-10.3\pm1.1$ on the truncated probe set and $-11.2\pm1.2$ on an exact-sampling probe set; under exact training samples it is $-0.41\pm0.14$ and $-0.48\pm0.19$.
For 2B, the corresponding values are $-10.7\pm0.4$ and $-11.2\pm0.4$ (truncated) and $-0.03\pm0.30$ and $-0.01\pm0.32$ (exact).

\begin{table}[htbp]
\caption{Dose-response on 2B under truncated sampling (\cref{fig:sign}a): mean over 3 seeds.}
\label{tab:dose}
\centering
\small
\begin{tabular}{@{}lrrrrrrr@{}}
\toprule
$\lambda$ & $-1$ & $-0.5$ & $0$ & $0.5$ & $1$ & $1.5$ & $2$ \\
\midrule
Mean $\net$ & $-0.157$ & $-0.121$ & $-0.080$ & $-0.037$ & $0$ & $+0.034$ & $+0.066$ \\
Entropy slope ($10^{-3}$ nats/step) & $+18.8$ & $+15.9$ & $+11.9$ & $+6.9$ & $+2.4$ & $-1.5$ & $-4.1$ \\
\bottomrule
\end{tabular}
\end{table}

\begin{table}[htbp]
\caption{Temperature mismatch on 2B (\cref{fig:sign}b): entropy slope ($10^{-3}$ nats/step) per seed, with the truncated and exact means of \cref{tab:e4} for reference.}
\label{tab:temp}
\centering
\small
\begin{tabular}{@{}lrrrr@{}}
\toprule
Sampler & \MASK{} & \FAIL+\NZ{} & \FAIL{} & \MASK$-$\FAIL{} \\
\midrule
$T=0.7$ & $-13.6$, $-13.4$ & $+4.9$, $+3.7$ & $+21.5$, $+19.1$ & $-35.2$, $-32.6$ \\
$T=1.2$ & $+20.3$, $+21.4$, $+22.8$ & $-7.3$, $-5.8$, $-6.8$ & $-12.1$, $-11.1$, $-11.5$ & $+32.4$, $+32.5$, $+34.3$ \\
Truncated (mean) & $-9.38$ & $+2.43$ & $+11.94$ & $-21.3$ \\
Exact (mean) & $-0.38$ & $-0.95$ & $-1.08$ & $+0.70$ \\
\bottomrule
\end{tabular}
\end{table}

\paragraph{9B arms.}
Batch accuracy is computed on qualifying groups only (at least one correct and one failed response), which compresses differences between arms. Mean net advantage, over-length rate, and batch accuracy are $-0.070$, 0.19, and 0.67 for truncated \FAIL+\PEN{}; $0$, 0.20, and 0.67 with \NZ{}; $-0.090$, 0.15, and 0.68 for exact \FAIL+\PEN{}; and $0$, 0.14, and 0.69 with \NZ{}.
The single-seed \MASK{} runs have $\net=+0.155$ (truncated) and $+0.168$ (exact); probe entropy falls from 0.496 to 0.430 under truncation and rises to 0.526 under exact sampling; under truncation this reproduces the 2B ordering \MASK{} $<$ \FAIL+\NZ{} $<$ \FAIL{}.
The renormalized learner has $-0.077$, 0.20, and 0.67; with the optimizer state reset, \FAIL+\PEN{} has $-0.072$, 0.18, and 0.68 and \NZ{} has $0$, 0.23, and 0.66 (two seeds each).
With the reset, held-out accuracy is 76.0\% for \FAIL+\PEN{} and 83.5\% for \NZ{}, so \NZ{} exceeds \FAIL+\PEN{} by 6.3 and 8.7 points in the two seeds, against 8.0 and 11.7 with the inherited state; probe entropy changes by $+0.206$ and $-0.042$ nats. The reset \NZ{} runs lose entropy, unlike the inherited ones ($+0.016$); we did not test why.

\begin{table*}[htbp]
\caption{Per-benchmark accuracy (\%) of the 9B model after 10 steps from the fork, mean over 4 seeds; 95\% intervals over seeds are given in the next paragraph. In-distribution sets use an 8{,}192-token limit (held-out: 300 problems $\times1$; fixed: 32 problems $\times8$); public benchmarks use 16{,}384 tokens (MATH-500 $\times1$; AIME and AMC $\times4$). All evaluations use the truncated sampler. $\Delta H$: change of probe-set entropy (nats). Differences are computed before rounding. Renorm.: truncated sampling with the learner renormalized over the sampled support. The fork point has 78.7\% held-out accuracy.}
\label{tab:e6full}
\vskip 0.05in
\centering
\footnotesize
\begin{tabular}{@{}llrrrrrrr@{}}
\toprule
Sampler & Rule & Held-out & Fixed & MATH-500 & AIME24 & AIME25 & AMC23 & $\Delta H$ \\
\midrule
Truncated & \FAIL+\PEN{} & 75.3 & 72.3 & 85.0 & 36.9 & 32.1 & 71.7 & $+0.215$ \\
 & \FAIL+\PEN+\NZ{} & 83.5 & 83.6 & 92.4 & 48.8 & 41.7 & 87.5 & $+0.016$ \\
 & \emph{difference} & $\mathbf{+8.2}$ & $\mathbf{+11.3}$ & $+7.5$ & $+11.9$ & $+9.6$ & $+15.8$ & \\
\midrule
Exact & \FAIL+\PEN{} & 84.0 & 83.3 & 92.2 & 49.6 & 40.2 & 84.4 & $+0.007$ \\
 & \FAIL+\PEN+\NZ{} & 84.2 & 82.7 & 91.8 & 50.6 & 40.6 & 82.2 & $+0.009$ \\
 & \emph{difference} & $+0.2$ & $-0.6$ & $-0.3$ & $+1.0$ & $+0.4$ & $-2.2$ & \\
\midrule
Renorm. & \FAIL+\PEN{} & 84.2 & 81.3 & 91.5 & 50.0 & 39.4 & 84.7 & $+0.035$ \\
\bottomrule
\end{tabular}
\vskip -0.1in
\end{table*}

\paragraph{9B intervals and single-token positions.}
The 95\% intervals of the \NZ$-$\FAIL+\PEN{} differences in \cref{tab:e6full}, with the seed as the unit, are $[4.2,12.1]$ (held-out), $[5.9,16.8]$ (fixed), $[2.2,12.7]$ (MATH-500), $[2.4,21.4]$ (AIME24), $[4.6,14.6]$ (AIME25), and $[4.7,26.9]$ (AMC23) under truncated sampling, and $[-0.5,1.0]$, $[-7.5,6.3]$, $[-3.2,2.6]$, $[-6.3,8.4]$, $[-2.7,3.6]$, and $[-3.9,-0.5]$ under exact sampling; the fixed-set interaction is $+11.9$ ($[0.6,23.2]$, $p=0.04$).
Positions where the top token alone exceeds the top-$p$ threshold give a deterministic update but little drift, since the gap is $O(\varepsilon H)$ for $\pol_1=1-\varepsilon$; on 9B rollouts they are 60\% of all positions and carry 12\% of the gap.

\begin{figure}[htbp]
\centering
\includegraphics[width=\textwidth]{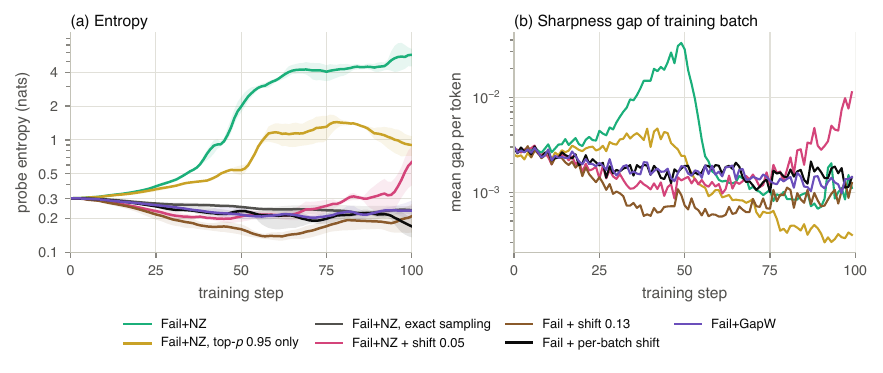}
\caption{100-step controls on 2B under top-$k$ 20 and top-$p$ 0.95 unless noted (mean over two or three seeds; \cref{tab:long}), with the truncated \FAIL+\NZ{} runs of \cref{fig:long}. \textbf{(a)} Probe-set entropy. \textbf{(b)} Mean gap per token of the training batch (none under exact sampling).}
\label{fig:controls}
\end{figure}

\paragraph{Controls and short runs.}
Under top-$p$ 0.95 alone the batch gap rises from $2.4\times10^{-3}$ to $5.2\times10^{-3}$ by steps 39--42, entropy peaks at 1.40 and 1.68 at steps 83 and 76, and it ends at 0.72 and 1.08; with top-$k$ 20 and top-$p$ 0.95 the gap rises to $(4$--$7)\times10^{-2}$ and entropy to 4.5--7.3.
With the fixed shift, the gap rises from $2.7\times10^{-3}$ to $1.7\times10^{-2}$ in seed 0 and to $6.4\times10^{-3}$ in seed 1 by the end of the run.
Over 15 steps under truncated sampling, \GW{} has slopes $-1.80$, $-0.83$, and $-1.84$ ($\net=+0.04$ over 15 steps, $+0.05$ over 100); under sequence-mean aggregation, \FAIL{} has $+1.89$ and $+2.09$ ($\net=0$; $+11.8$ and $+13.0$ with prompt-mean aggregation, predicted before the runs to fall to the level of \FAIL+\NZ{}) and \MASK{} $-8.93$ and $-9.89$ ($\net=+0.33$, $+0.30$).
With standard-deviation normalization, \MASK{} has $-9.2$ and $-8.9$ and \FAIL{} $+11.0$ and $+8.3$ under truncation, against $-1.7$, $+0.1$ and $-1.8$, $-1.2$ under exact sampling.
With four clipped minibatch updates per batch, where the old policy is fixed at the start of the batch so that ratios depart from one after the first update (largest log-ratio about 0.5), \MASK{} has $-15.2$ and $-14.4$ and \FAIL{} $+123.8$ and $+140.0$ under truncation, against $-3.7$, $-4.6$ and $-3.1$, $-1.3$ under exact sampling.
Shifting the advantages of \MASK{} by $-\lambda\net_g$ with $\lambda=0.5$, $1$, $1.5$ gives $\net=+0.15$, $0$, $-0.24$ (means over 15 steps; at $\lambda=1.5$ the first step has $-0.19=-0.5\times0.38$, and $|\net|$ grows as more responses go unresolved, from 41--43\% to 70--88\%) and slopes $-8.6$ and $-8.5$, $+0.7$ and $+0.2$, $+17.5$ and $+17.1$ ($-9.8$ and $-9.7$ at $\lambda=0$); against the line of \cref{fig:drift}, fitted before these runs, the residuals range from $-3.9$ to $+0.5$.
With SGD instead of Adam, at learning rate 0.007, chosen so that the 15-step slope of \MASK{} matches Adam's, \MASK{} has slopes $-8.8$ and $-7.9$ ($\net=+0.34$ and $+0.31$, logged drift $-58$ and $-49\times10^{-5}$) and \MASK{} shifted with $\lambda=0.5$ has $-4.3$ ($\net=+0.15$, drift $-31\times10^{-5}$, one seed), so the response is proportional to the drift; learning rates of 0.03 and 0.1 are already nonlinear in a single step, and a single step at 4 or more raises probe entropy by 5--9 nats.

\paragraph{Second model family.}
On OLMo-3-7B-Instruct we grade the last boxed answer, because strict grading counts 34\% of responses as format failures, use a 2{,}048-token limit, and probe entropy on 64 of its own responses to held-out prompts. The \MASK$-$\FAIL{} slope difference is $-2.7$ and $-2.9\times10^{-3}$ nats per step under truncation and $+0.1$ and $0.0$ under exact sampling, about one eighth of the $-21.3$ of Qwen3.5-2B, although the sharpness gap ($2.6\times10^{-3}$) and the logged drift are similar in size. OLMo's probe entropy moves less in the \MASK{} and \FAIL{} arms (0.01--0.03 nats over 15 steps, against 0.13--0.18 for Qwen3.5-2B).

\paragraph{Long runs.}
\Cref{tab:long} lists the 100-step runs of \cref{sec:exp-long} by seed. A node failure cut the first runs of seeds 0 and 1 of \FAIL{}, \FAIL+\NZ{}, \MASK{}, and the renormalized learner at steps 43--78; we reran them from scratch, and the reruns replace them; up to the cut they agree with the cut runs, crossing 0.5 nats within three steps of them.
\GW{} uses the gap that the learner computes for each sampled response under the sampling rule before the update; for \GW{} the logged $D_{\mathrm{net}}+D_{\mathrm{cov}}$ is zero at every step.
The surprisal-weighted variant replaces $\bar\gap_i$ in the weights by the response's mean negative log-probability under the sampler's raw log-probabilities, $\ell_i=-\frac1{L_i}\sum_t\log\pol(y_{i,t}\mid s_{i,t})$, so that $b_g=\sum_i\omega_iL_i\ell_ir_i/\sum_i\omega_iL_i\ell_i$.
The two runs with a fixed shift of $+0.13$ ran on single H200 GPUs (the same Hopper architecture) after the original node was reclaimed; runs of \FAIL+\NZ{} with a shift of $+0.085$ and of the kicked \GW{} to 100 steps were cut at steps 38--42 by the same reclaim and follow the corresponding reported runs up to that point.
The per-batch shift adds to every advantage the constant $c_t=-\sum_i\omega_iL_iA_i\bar\gap_i/\sum_i\omega_iL_i\bar\gap_i$, summed over the whole batch (mean $+0.12$ to $+0.14$ over the run).
The kick divides the scale $1+w$ of the final RMSNorm by 1.2, which divides every logit by 1.2; the sampler loads the same weights, and we checked that the logits change by the factor $1/1.2$ to within $0.5\%$.
On the 15-step truncated batches, $\bar\gap_i$ and $\ell_i$ have correlation 0.97--0.98 across responses and 0.91--0.93 within groups, and recomputing the drift with these weights removes 96--99\% of the gap covariance left by \NZ{}; in the long surprisal-weighted runs the logged $D_{\mathrm{net}}+D_{\mathrm{cov}}$ is 5--15\% of $D_{\mathrm{cov}}$ (means over 25-step windows).

\begin{table}[htbp]
\caption{Long 2B runs, by seed, under truncated sampling unless noted. Entropy: probe-set entropy at the last step (initial 0.303; kicked runs 0.402). Held-out accuracy at steps 25 / 50 / 75 / 100 (initial 44--47\%; kicked runs: steps 25 / 50, initial 41--43\%). Peak entropy of the truncated \FAIL+\NZ{} runs: 4.53, 6.45, and 7.30. \FAIL{} runs stop at step 50, once their entropy has saturated.}
\label{tab:long}
\vskip 0.05in
\centering
\footnotesize
\begin{tabular}{@{}llrrl@{}}
\toprule
Rule & Seed & Last step & Entropy & Held-out accuracy (\%) \\
\midrule
\FAIL{} & 0 / 1 / 2 & 50 & 4.44 / 4.95 / 5.52 & 26.0 / 0.0;\; 28.0 / 0.0;\; 30.3 / 0.0 \\
\FAIL+\NZ{} & 0 / 1 & 100 / 100 & 4.50 / 6.45 & 56.0 / 31.0 / 36.7 / 37.3;\; 56.3 / 28.0 / 26.0 / 16.0 \\
 & 2 & 100 & 6.21 & 54.7 / 10.0 / 29.0 / 13.3 \\
\MASK{} & 0 / 1 & 100 / 100 & 0.070 / 0.072 & 66.3 / 68.7 / 73.0 / 65.0;\; 66.3 / 65.0 / 65.3 / 60.7 \\
 & 2 & 100 & 0.068 & 61.3 / 61.7 / 65.7 / 68.7 \\
Renormalized & 0 / 1 & 100 / 100 & 0.262 / 0.213 & 58.0 / 61.3 / 67.7 / 69.0;\; 53.7 / 62.0 / 68.7 / 69.3 \\
\FAIL+\NZ{}, exact & 0 / 1 & 100 / 100 & 0.259 / 0.217 & 61.7 / 66.3 / 63.7 / 71.3;\; 55.0 / 65.0 / 70.7 / 68.0 \\
\FAIL+\NZ{}, top-$p$ 0.95 & 0 / 1 & 100 / 100 & 0.723 / 1.078 & 56.7 / 63.3 / 57.0 / 64.3;\; 49.7 / 62.0 / 50.3 / 56.7 \\
\FAIL+\NZ{}, shift $+0.05$ & 0 / 1 & 100 / 100 & 0.873 / 0.395 & 63.7 / 67.0 / 68.0 / 33.3;\; 60.7 / 66.7 / 68.0 / 64.0 \\
\FAIL+\GW{} & 0 & 100 & 0.216 & 59.0 / 65.3 / 70.0 / 70.0 \\
 & 1 & 100 & 0.202 & 63.3 / 70.0 / 67.3 / 71.3 \\
 & 2 & 100 & 0.288 & 65.3 / 66.7 / 71.0 / 64.7 \\
Surprisal-weighted & 0 & 100 & 0.255 & 60.0 / 61.3 / 67.7 / 70.0 \\
 & 1 & 100 & 0.175 & 61.3 / 48.7 / 65.3 / 69.0 \\
 & 2 & 100 & 0.408 & 63.0 / 68.0 / 65.7 / 63.3 \\
\FAIL{} + shift $+0.13$ (H200) & 0 / 1 & 100 / 100 & 0.150 / 0.269 & 62.0 / 66.3 / 67.3 / 70.7;\; 63.7 / 66.7 / 68.0 / 67.7 \\
\FAIL{} + per-batch shift & 0 & 100 & 0.172 & 59.7 / 64.3 / 71.7 / 71.7 \\
 & 1 & 100 & 0.136 & 62.3 / 67.0 / 71.0 / 69.3 \\
 & 2 & 100 & 0.204 & 56.7 / 68.3 / 66.0 / 70.3 \\
\midrule
Kicked, \FAIL+\NZ{} & 0 / 1 & 50 / 50 & 5.37 / 6.01 & 35.7 / 19.7;\; 50.3 / 21.3 \\
Kicked, \FAIL+\NZ{}, exact & 0 / 1 & 50 / 50 & 0.311 / 0.298 & 58.3 / 64.0;\; 61.0 / 68.0 \\
Kicked, \FAIL+\GW{} & 0 / 1 & 50 / 50 & 0.351 / 0.319 & 62.7 / 67.0;\; 61.0 / 56.0 \\
\bottomrule
\end{tabular}
\end{table}

\section{Survey Details}
\label{app:survey}

We read source code at the commits listed below, together with each project's documentation, issues, and pull requests, without running any of it.
The inference engines are vLLM, SGLang, and Hugging Face transformers.
A pipeline is counted as corrected only if it samples without truncation, renormalizes the learner over the sampled support, or centers the score on it; importance-sampling corrections do not count (\cref{prop:fixes}d).
For temperature, a learner that divides logits by the sampling temperature has no mismatch.

\begin{table}[htbp]
\caption{RL frameworks at the listed commits: default rollout sampler, whether the learner applies the sampling temperature, and correction for truncation.}
\label{tab:survey}
\centering
\footnotesize
\setlength{\tabcolsep}{4pt}
\begin{tabular}{@{}p{0.25\textwidth}p{0.2\textwidth}p{0.15\textwidth}p{0.34\textwidth}@{}}
\toprule
Framework (commit) & Default $T$ / top-$p$ / top-$k$ & Learner divides by $T$ & Truncation correction \\
\midrule
verl (d430157) & 1 / 1 / off & yes & score centering on support, optional (off) \\
OpenRLHF (dc2a7ad) & 1 / 1 / off & yes & none \\
TRL GRPOTrainer (e2dc408) & 1 / 1 / off & yes & none \\
TRL AsyncGRPOTrainer (e2dc408) & 1 / 1 / off & yes & mask replay, automatic (top-$k$/top-$p$) \\
Unsloth (71d24f7) & 1 / 1 / off & yes & recomputed support, optional (off) \\
NeMo-RL (f0639c6) & 1 / 1 / off & yes & recomputed support, automatic (no min-$p$) \\
open-instruct (94df04d) & 1 / 1 / -- & yes & none \\
slime (18868a7) & 1 / 1 / off & yes & mask replay, automatic for top-$p$ only \\
Megatron-LM RL (b291235) & 1 / off / inherited$^\dagger$ & no & none \\
ROLL (d3506ea) & 0.95 / 0.7 / 50 & no & none \\
AReaL (f31bd36) & 1 / 1 / off & yes & none \\
prime-rl (1999806) & 1 / 1 / off & yes & mask replay, automatic; min-$p$ disallowed \\
SkyRL (cb3b450) & 1 / 1 / off & yes & mask replay, optional (off) \\
torchtitan RL (4103c72) & 0.8 / top-$p<1$ rejected / -- & yes (since 2026-10) & truncation disallowed \\
torchforge (b32f06b) & 1 / 1 / engine default & no & none \\
tinker-cookbook (54b3cb2) & 1 / 1 / off & not determined & none \\
OpenPipe ART (5f0c1fb) & user-defined & no & none \\
Atropos (296a0bc) & environment-defined & no & none \\
rLLM (4e828c8) & 1 / 1 / off & yes (verl backend) & none \\
\bottomrule
\end{tabular}
\\[2pt]{\footnotesize $^\dagger$The RL client omits top-$k$, which the serving endpoint fills from the model's generation config (top-$k$ 20 for Qwen3); inferred from code.}
\end{table}

\paragraph{Reports and recipes.}
We collected 123 technical reports, open recipes, and papers that describe RLVR training of language models, and classified each by its training rollout sampler: (A) truncation without correction, (B) truncation with support renormalization or score centering, (C) no truncation and matched temperature, (D) temperature not applied by the learner, and (E) not disclosed.
Eighteen entries were collected specifically from frameworks with truncating defaults and are excluded from the prevalence count.
The remaining 105 entries split into 5 A, 2 B, 46 C, 2 D, and 50 E; most large-model technical reports are E, so the 13\% rate among the 55 disclosed entries describes public recipes, not production pipelines.
Commonly cited sampling settings of open models (e.g., top-$p$ 0.95, top-$k$ 20) are evaluation settings and were not counted as training settings.
Earlier releases had truncating defaults that the table, which uses current commits, does not show: recipes from TRL v0.16--0.17, EasyR1 v0.3, and Unsloth's GRPO notebooks truncate with top-$k$ 50, top-$p$ 0.99, or min-$p$ 0.1.

\section{Score Centering Equals Support Renormalization: Numerical Check}
\label{app:sc}

The identity in \cref{prop:fixes}c follows from the definitions; this check verifies our implementations.
We implemented score centering from its definition, $f(y)=\log\pol(y)-\sum_v\mathrm{sg}(q_v)\log\pol(v)$ with $\mathrm{sg}$ the stop-gradient, and support renormalization as renormalization over the kept set, and installed both on the same small model.
On 4 sequences of 300 tokens with kept sets of size 1 to 19, the logit gradient of $f$ equals $e_y-q$ to within $6\times10^{-8}$, and the parameter gradients of the two implementations differ by a relative $3\times10^{-6}$, the level of float32 rounding.

\section{Stale Rollouts and Clipping}
\label{app:stale}

Asynchronous pipelines sample from an older policy $\pol_{\mathrm{old}}$.
We added a rollout lag of $k$ steps to the tabular simulation of \cref{app:numerics}, here with SGD on the logit table (learning rate 50), and compared vanilla policy gradient with importance ratios against $\pol_{\mathrm{old}}$ and PPO clipping.
\begin{itemize}
\item Without importance ratios, a constant shift pulls the policy back toward $\pol_{\mathrm{old}}$ when $\net>0$ and pushes it away when $\net<0$, in 99.5--100\% of steps. Its entropy effect is feedback on the most recent entropy trend rather than an independent source, so it has no universal sign at first order; over many steps with a fixed behavior policy, a baseline above the behavior value instead drives the limit policy to a deterministic one~\citep{arnal2025asymre}.
\item Importance ratios with clipping remove the stale component: under exact sampling the dose-response slope of entropy on $\net$ is $-0.002$ to $+0.005$ for lags 2--32. They leave the truncation component intact: under truncated sampling the slope is $-0.290$, $-0.301$, and $-0.246$ for lags 2, 8, and 32, against $-0.297$ without lag.
\end{itemize}
Ratios against $\pol_{\mathrm{old}}$ correct for staleness because $\pol_{\mathrm{old}}$ is the distribution that generated the tokens up to truncation; they cannot correct truncation because $\pol_{\mathrm{old}}$ is not the sampler.

Clipping has a related effect when a batch is used for several updates.
With ratio $\rho=\pol/\pol_{\mathrm{old}}$ and clip range $\epsilon$, the expected gradient of a constant advantage $c$ is $c\,\pol(U)(q_U-\pol)$, where $U$ is the set of tokens not clipped for the sign of $c$; clipping therefore acts like a truncation of the support to $U$, and a constant shift again moves entropy, with a sign that depends on which tokens are clipped~\citep{park2025clip}.
This offers one route by which shifting all advantages, as positive-only versus negative-only reinforcement does, changes entropy in pipelines with several clipped updates per batch~\citep{zhu2025nsr}, whereas our single-update exact-sampling runs show almost no effect of a shift of the same size (\cref{tab:probes}).

%% file: main.bbl
\begin{thebibliography}{42}
\providecommand{\natexlab}[1]{#1}
\providecommand{\url}[1]{\texttt{#1}}
\expandafter\ifx\csname urlstyle\endcsname\relax
  \providecommand{\doi}[1]{doi: #1}\else
  \providecommand{\doi}{doi: \begingroup \urlstyle{rm}\Url}\fi

\bibitem[Arnal et~al.(2025)Arnal, Narozniak, Cabannes, Tang, Kempe, and
  Munos]{arnal2025asymre}
Arnal, C., Narozniak, G., Cabannes, V., Tang, Y., Kempe, J., and Munos, R.
\newblock Asymmetric {REINFORCE} for off-policy reinforcement learning:
  Balancing positive and negative rewards.
\newblock \emph{arXiv preprint arXiv:2506.20520}, 2025.

\bibitem[Cao et~al.(2026)Cao, Yuan, and Xie]{cao2026anatomy}
Cao, H., Yuan, Y., and Xie, T.
\newblock Anatomy of off-policy policy gradient: Importance sampling, {KL}
  regularization, and baselines.
\newblock \emph{arXiv preprint arXiv:2610.25725}, 2026.

\bibitem[Chung et~al.(2021)Chung, Thomas, Machado, and
  Le~Roux]{chung2021beyond}
Chung, W., Thomas, V., Machado, M.~C., and Le~Roux, N.
\newblock Beyond variance reduction: Understanding the true impact of baselines
  on policy optimization.
\newblock In \emph{Proceedings of the 38th International Conference on Machine
  Learning (ICML)}, 2021.

\bibitem[{Cognition}(2026)]{cognition2026swe17}
{Cognition}.
\newblock {SWE-1.7}: Frontier intelligence at a fraction of the cost.
\newblock \url{https://cognition.com/blog/swe-1-7}, 2026.
\newblock Blog post, July 8, 2026.

\bibitem[Cui et~al.(2025)Cui, Zhang, Chen, Yuan, Wang, Zuo, Li, Fan, Chen,
  Chen, et~al.]{cui2025entropy}
Cui, G., Zhang, Y., Chen, J., Yuan, L., Wang, Z., Zuo, Y., Li, H., Fan, Y.,
  Chen, H., Chen, W., et~al.
\newblock The entropy mechanism of reinforcement learning for reasoning
  language models.
\newblock \emph{arXiv preprint arXiv:2505.22617}, 2025.

\bibitem[Decugis et~al.(2026)Decugis, O'Brien, Bach, Synnaeve, and
  Cohen]{fade2026}
Decugis, J., O'Brien, S., Bach, F., Synnaeve, G., and Cohen, T.
\newblock Don't let gains {FADE}: Breaking down policy gradient weights in
  {RL}.
\newblock \emph{arXiv preprint arXiv:2607.01490}, 2026.

\bibitem[{DeepSeek-AI}(2025)]{deepseek2025v32}
{DeepSeek-AI}.
\newblock {DeepSeek-V3.2}: Pushing the frontier of open large language models.
\newblock \emph{arXiv preprint arXiv:2512.02556}, 2025.

\bibitem[Fang \& Khazi(2025)Fang and Khazi]{fang2025praxis}
Fang, Q. and Khazi, D.
\newblock Mismatch praxis: Rollout settings and {IS} corrections.
\newblock \url{https://www.paperinstruments.com/blog/mismatch-praxis}, 2025.
\newblock Paper Instruments blog, December 2, 2025.

\bibitem[Greensmith et~al.(2004)Greensmith, Bartlett, and
  Baxter]{greensmith2004variance}
Greensmith, E., Bartlett, P.~L., and Baxter, J.
\newblock Variance reduction techniques for gradient estimates in reinforcement
  learning.
\newblock \emph{Journal of Machine Learning Research}, 5:\penalty0 1471--1530,
  2004.

\bibitem[Han \& Hu(2026)Han and Hu]{han2026baselines}
Han, S. and Hu, Z.
\newblock Baselines and score centering in policy gradient estimation.
\newblock \emph{arXiv preprint arXiv:2610.22263}, 2026.

\bibitem[Hao et~al.(2025)Hao, Dong, Wu, Huang, Chi, and Wei]{hao2025opo}
Hao, Y., Dong, L., Wu, X., Huang, S., Chi, Z., and Wei, F.
\newblock On-policy {RL} with optimal reward baseline.
\newblock \emph{arXiv preprint arXiv:2505.23585}, 2025.

\bibitem[Hendrycks et~al.(2021)Hendrycks, Burns, Kadavath, Arora, Basart, Tang,
  Song, and Steinhardt]{hendrycks2021math}
Hendrycks, D., Burns, C., Kadavath, S., Arora, A., Basart, S., Tang, E., Song,
  D., and Steinhardt, J.
\newblock Measuring mathematical problem solving with the {MATH} dataset.
\newblock In \emph{Advances in Neural Information Processing Systems, Datasets
  and Benchmarks Track}, 2021.

\bibitem[{Hugging Face TRL contributors}(2026)]{trl7627}
{Hugging Face TRL contributors}.
\newblock {AsyncGRPO}: Renormalize log-probs over the top-$k$/top-$p$ candidate
  set.
\newblock \url{https://github.com/huggingface/trl/pull/7627}, 2026.
\newblock GitHub pull request, merged October 16, 2026.

\bibitem[Jiang(2026)]{modal2026mixed}
Jiang, N.
\newblock Toward {RL} on mixed-precision, mixed-accelerator samplers.
\newblock
  \url{https://modal.com/blog/toward-rl-on-mixed-precision-mixed-accelerator-samplers},
  2026.
\newblock Modal blog, October 13, 2026.

\bibitem[Kwon et~al.(2023)Kwon, Li, Zhuang, Sheng, Zheng, Yu, Gonzalez, Zhang,
  and Stoica]{kwon2023vllm}
Kwon, W., Li, Z., Zhuang, S., Sheng, Y., Zheng, L., Yu, C.~H., Gonzalez, J.~E.,
  Zhang, H., and Stoica, I.
\newblock Efficient memory management for large language model serving with
  {PagedAttention}.
\newblock In \emph{Proceedings of the 29th Symposium on Operating Systems
  Principles (SOSP)}, 2023.

\bibitem[Le~Roux(2016)]{leroux2016efficient}
Le~Roux, N.
\newblock Efficient iterative policy optimization.
\newblock \emph{arXiv preprint arXiv:1612.08967}, 2016.

\bibitem[Le~Roux et~al.(2025)Le~Roux, Bellemare, Lebensold, Bergeron, Greaves,
  Fr{\'e}chette, Pelletier, Thibodeau-Laufer, Toth, and Work]{leroux2025topr}
Le~Roux, N., Bellemare, M.~G., Lebensold, J., Bergeron, A., Greaves, J.,
  Fr{\'e}chette, A., Pelletier, C., Thibodeau-Laufer, E., Toth, S., and Work,
  S.
\newblock Tapered off-policy {REINFORCE}: Stable and efficient reinforcement
  learning for {LLMs}.
\newblock \emph{arXiv preprint arXiv:2503.14286}, 2025.

\bibitem[Li et~al.(2026{\natexlab{a}})Li, Wang, Ding, Lochab, Grama, and
  Zhang]{entrocraft2026}
Li, B., Wang, Y., Ding, Y., Lochab, A., Grama, A., and Zhang, R.
\newblock Addressing performance saturation for {LLM} {RL} via precise entropy
  curve control.
\newblock \emph{arXiv preprint arXiv:2604.26326}, 2026{\natexlab{a}}.

\bibitem[Li et~al.(2026{\natexlab{b}})Li, Xu, Li, Liu, Liu, Tong, Zheng, Xue,
  Zhang, Cai, Zhang, Liu, and Wang]{li2026otb}
Li, Y., Xu, J., Li, Z., Liu, J., Liu, W., Tong, Y., Zheng, L., Xue, Z., Zhang,
  Y., Cai, T., Zhang, G., Liu, Q., and Wang, B.
\newblock The optimal token baseline: Variance reduction for long-horizon
  {LLM-RL}.
\newblock In \emph{International Conference on Machine Learning},
  2026{\natexlab{b}}.
\newblock arXiv:2602.07078.

\bibitem[Liu et~al.(2025)Liu, Chen, Li, Qi, Pang, Du, Lee, and
  Lin]{liu2025drgrpo}
Liu, Z., Chen, C., Li, W., Qi, P., Pang, T., Du, C., Lee, W.~S., and Lin, M.
\newblock Understanding {R1-Zero}-like training: A critical perspective.
\newblock \emph{arXiv preprint arXiv:2503.20783}, 2025.

\bibitem[Luo et~al.(2026)Luo, Sun, Wu, Xu, Deng, Hu, and Tang]{stare2026}
Luo, H., Sun, Q., Wu, S., Xu, C., Deng, W., Hu, H., and Tang, Y.
\newblock {STARE}: Surprisal-guided token-level advantage reweighting for
  policy entropy stability.
\newblock \emph{arXiv preprint arXiv:2606.19236}, 2026.

\bibitem[Marek \& Ryabinin(2026)Marek and Ryabinin]{marek2026score}
Marek, M. and Ryabinin, M.
\newblock Score centering stabilizes off-policy reinforcement learning.
\newblock \emph{arXiv preprint arXiv:2609.20807}, 2026.

\bibitem[Pan(2026)]{pan2026cool}
Pan, T.
\newblock Cool the sampler, not the learner: Sampling temperature moves the
  staleness cliff of importance-corrected {GRPO}.
\newblock \emph{arXiv preprint arXiv:2609.36953}, 2026.

\bibitem[Park et~al.(2025)Park, Kim, Kim, Jo, Choi, Cho, and Ryu]{park2025clip}
Park, J.~R., Kim, J., Kim, G., Jo, J., Choi, S., Cho, J., and Ryu, E.~K.
\newblock Clip-low increases entropy and clip-high decreases entropy in
  reinforcement learning of large language models.
\newblock \emph{arXiv preprint arXiv:2509.26114}, 2025.

\bibitem[Petrenko et~al.(2026)Petrenko, Lipkin, Chen, Wijmans, Cusumano-Towner,
  Giryes, and Kr{\"a}henb{\"u}hl]{petrenko2026repo}
Petrenko, A., Lipkin, B., Chen, K., Wijmans, E., Cusumano-Towner, M., Giryes,
  R., and Kr{\"a}henb{\"u}hl, P.
\newblock Entropy-preserving reinforcement learning.
\newblock In \emph{International Conference on Learning Representations}, 2026.
\newblock arXiv:2603.11682.

\bibitem[{PyTorch torchtitan contributors}(2026)]{torchtitan5021}
{PyTorch torchtitan contributors}.
\newblock [rl] reject top\_p $<$ 1 until the trainer replays truncated
  sampling.
\newblock \url{https://github.com/pytorch/torchtitan/pull/5021}, 2026.
\newblock GitHub pull request, merged October 4, 2026.

\bibitem[Shao et~al.(2024)Shao, Wang, Zhu, Xu, Song, Bi, Zhang, Zhang, Li, Wu,
  and Guo]{shao2024deepseekmath}
Shao, Z., Wang, P., Zhu, Q., Xu, R., Song, J., Bi, X., Zhang, H., Zhang, M.,
  Li, Y.~K., Wu, Y., and Guo, D.
\newblock {DeepSeekMath}: Pushing the limits of mathematical reasoning in open
  language models.
\newblock \emph{arXiv preprint arXiv:2402.03300}, 2024.

\bibitem[{Team Olmo}(2025)]{olmo2025olmo3}
{Team Olmo}.
\newblock Olmo 3.
\newblock \emph{arXiv preprint arXiv:2512.13961}, 2025.

\bibitem[{The Microsoft AI Team}(2026)]{mai2026thinking}
{The Microsoft AI Team}.
\newblock {MAI-Thinking-1}: Building a hill-climbing machine.
\newblock Technical report, Microsoft AI, 2026.
\newblock \url{https://microsoft.ai/pdf/mai-thinking-1.pdf}.

\bibitem[Wang et~al.(2025)Wang, Li, Bai, Deng, Dong, Lan, and
  Wang]{scoperl2025}
Wang, C., Li, Z., Bai, J., Deng, H., Dong, G., Lan, G., and Wang, Y.
\newblock {SCOPE-RL}: Stable and quantitative control of policy entropy in {RL}
  post-training.
\newblock \emph{arXiv preprint arXiv:2510.08141}, 2025.

\bibitem[Wang et~al.(2026)Wang, Xie, Zhang, Sun, Chen, Li, and
  Zhang]{wang2026entropydyn}
Wang, S., Xie, Y., Zhang, W., Sun, Y., Chen, Y., Li, Y., and Zhang, Y.
\newblock On the entropy dynamics in reinforcement fine-tuning of large
  language models.
\newblock \emph{arXiv preprint arXiv:2602.03392}, 2026.

\bibitem[Williams(1992)]{williams1992reinforce}
Williams, R.~J.
\newblock Simple statistical gradient-following algorithms for connectionist
  reinforcement learning.
\newblock \emph{Machine Learning}, 8\penalty0 (3--4):\penalty0 229--256, 1992.

\bibitem[Wu et~al.(2025)Wu, Huang, Wu, Zhang, Wang, and He]{qae2025}
Wu, J., Huang, K., Wu, J., Zhang, A., Wang, X., and He, X.
\newblock Quantile advantage estimation: Stabilizing {RLVR} for {LLM}
  reasoning.
\newblock \emph{arXiv preprint arXiv:2509.22611}, 2025.

\bibitem[{Xiaomi LLM-Core Team}(2026)]{xiaomi2026mimo}
{Xiaomi LLM-Core Team}.
\newblock {MiMo-V2.6}: Scaling reinforcement learning towards self-improvement.
\newblock \emph{arXiv preprint arXiv:2610.11959}, 2026.

\bibitem[Yang et~al.(2025)Yang, Xu, Chen, Liu, Lyu, Lin, Ye, and
  Yang]{yang2025entropic}
Yang, K., Xu, X., Chen, Y., Liu, W., Lyu, J., Lin, Z., Ye, D., and Yang, S.
\newblock {EntroPIC}: Towards stable long-term training of {LLMs} via entropy
  stabilization with proportional-integral control.
\newblock \emph{arXiv preprint arXiv:2511.15248}, 2025.

\bibitem[Yao et~al.(2025)Yao, Liu, Zhang, Dong, Shang, and Gao]{yao2025tis}
Yao, F., Liu, L., Zhang, D., Dong, C., Shang, J., and Gao, J.
\newblock Your efficient {RL} framework secretly brings you off-policy {RL}
  training.
\newblock \url{https://fengyao.notion.site/off-policy-rl}, 2025.
\newblock Blog post.

\bibitem[Yu et~al.(2025)Yu, Zhang, Zhu, Yuan, Zuo, Yue, Dai, Fan, Liu, Liu,
  et~al.]{yu2025dapo}
Yu, Q., Zhang, Z., Zhu, R., Yuan, Y., Zuo, X., Yue, Y., Dai, W., Fan, T., Liu,
  G., Liu, L., et~al.
\newblock {DAPO}: An open-source {LLM} reinforcement learning system at scale.
\newblock \emph{arXiv preprint arXiv:2503.14476}, 2025.

\bibitem[Zeng et~al.(2026)Zeng, Huang, Yin, Liu, Li, Li, Wu, Zheng, Zhang,
  Huang, and Qiu]{balanced2026}
Zeng, Z., Huang, J., Yin, Z., Liu, J., Li, Z., Li, B., Wu, Y., Zheng, Y.,
  Zhang, G., Huang, W., and Qiu, X.
\newblock Balanced aggregation: Understanding and fixing aggregation bias in
  {GRPO}.
\newblock \emph{arXiv preprint arXiv:2605.04077}, 2026.

\bibitem[Zhang et~al.(2026)Zhang, Fu, Shen, Zhao, Zhang, Xi,
  et~al.]{zhang2026polarity}
Zhang, J., Fu, Z., Shen, J., Zhao, Y., Zhang, Y., Xi, Z., et~al.
\newblock Entropy polarity in reinforcement fine-tuning: Direction, asymmetry,
  and control.
\newblock \emph{arXiv preprint arXiv:2605.11775}, 2026.

\bibitem[Zhang et~al.(2025)Zhang, Yuan, Huang, You, Hu, Ruan, Jian, Chen, and
  Hu]{zhang2025aer}
Zhang, X., Yuan, X., Huang, D., You, W., Hu, C., Ruan, J., Jian, A., Chen, K.,
  and Hu, X.
\newblock Revisiting entropy regularization: Adaptive coefficient unlocks its
  potential for {LLM} reinforcement learning.
\newblock \emph{arXiv preprint arXiv:2510.10959}, 2025.

\bibitem[Zheng et~al.(2025)Zheng, Liu, Li, Chen, Yu, Gao, Dang, Liu, Men, Yang,
  Zhou, and Lin]{zheng2025gspo}
Zheng, C., Liu, S., Li, M., Chen, X.-H., Yu, B., Gao, C., Dang, K., Liu, Y.,
  Men, R., Yang, A., Zhou, J., and Lin, J.
\newblock Group sequence policy optimization.
\newblock \emph{arXiv preprint arXiv:2507.18071}, 2025.

\bibitem[Zhu et~al.(2025)Zhu, Xia, Wei, Chen, Chen, and Meng]{zhu2025nsr}
Zhu, X., Xia, M., Wei, Z., Chen, W.-L., Chen, D., and Meng, Y.
\newblock The surprising effectiveness of negative reinforcement in {LLM}
  reasoning.
\newblock In \emph{Advances in Neural Information Processing Systems
  (NeurIPS)}, 2025.
\newblock arXiv:2506.01347.

\end{thebibliography}
